\documentclass{article}

\PassOptionsToPackage{round}{natbib}
\usepackage[preprint]{neurips_2026}
\renewcommand{\cite}{\citep}

\usepackage[utf8]{inputenc}
\usepackage[T1]{fontenc}
\usepackage{hyperref}
\hypersetup{
  pdftitle={WarmPrior: Straightening Flow-Matching Policies with Temporal Priors},
  pdfauthor={Sinjae Kang and Chanyoung Kim and Kaixin Wang and Li Zhao and Kimin Lee}
}
\usepackage{url}
\usepackage{booktabs}
\usepackage{amsfonts}
\usepackage{nicefrac}
\usepackage{microtype}
\usepackage{xcolor}

\usepackage{graphicx}
\usepackage{subcaption}
\usepackage{wrapfig}
\usepackage{makecell}
\usepackage{amsmath}
\usepackage{amssymb}
\usepackage{mathtools}
\usepackage{amsthm}
\usepackage{algorithm}
\usepackage{algorithmic}
\usepackage[capitalize,noabbrev]{cleveref}

\theoremstyle{plain}
\newtheorem{theorem}{Theorem}[section]
\newtheorem{proposition}[theorem]{Proposition}

\theoremstyle{definition}

\theoremstyle{remark}

\definecolor{goodgreen}{RGB}{20,120,55}
\newcommand{\dgood}[1]{(\textcolor{goodgreen}{+#1})}
\newcommand{\dpos}[1]{(+#1)}
\newcommand{\dneg}[1]{($-$#1)}

\title{WarmPrior: Straightening Flow-Matching Policies with Temporal Priors}

\author{%
  Sinjae Kang$^{1}$ \quad Chanyoung Kim$^{1}$ \quad Kaixin Wang$^{2}$ \quad Li Zhao$^{2}$ \quad Kimin Lee$^{1}$ \\
  $^{1}$KAIST \quad $^{2}$Microsoft Research
}

\begin{document}

\maketitle

\begin{abstract}
  Generative policies based on diffusion and flow matching have become a dominant paradigm for visuomotor robotic control. We show that replacing the standard Gaussian source distribution with \emph{WarmPrior}, a simple temporally grounded prior constructed from readily available recent action history, consistently improves success rates on robotic manipulation tasks. We trace this gain to markedly \emph{straighter} probability paths, echoing the effect of optimal-transport couplings in Rectified Flow. Beyond standard behavior cloning, \emph{WarmPrior} also reshapes the exploration distribution in prior-space reinforcement learning, improving both sample efficiency and final performance. Collectively, these results identify the \emph{source distribution} as an important and underexplored design axis in generative robot control.
  Project page: \url{https://sinnnj.github.io/WarmPrior/}.
\end{abstract}

\section{Introduction}
\label{sec:intro}

Learning generative policies for robotic manipulation, such as diffusion policies and flow-matching policies, has become a dominant paradigm for multi-modal behavior cloning~\cite{chi2023diffusionpolicy, nvidia2025gr00tn1, black2024pi0}.
In these frameworks, a neural field transports samples from a fixed source distribution to the data manifold of action chunks.
Almost universally, this source distribution is the isotropic Gaussian $\mathcal{N}(0,I)$, a convention inherited from diffusion's denoising-from-noise interpretation~\cite{ho2020ddpm,song2021ddim} and preserved by flow matching~\cite{braun2024rfmp,hu2024adaflow} and its few-step policy descendants~\cite{prasad2024consistencypolicy,lu2024manicm,wang2024onedp}, while progress was pushed through the network, the interpolant, and the integrator.
The \emph{prior space} has been quietly left untouched.
Yet as denoising schedules shorten, the starting point absorbs more of the burden that integration steps once carried.
A stateless, uninformative source remains blind to the continuous, temporally correlated nature of robotic motion, forcing the policy to rebuild every action chunk from scratch.

\begin{figure}[t]
  \centering
  \includegraphics[width=\textwidth]{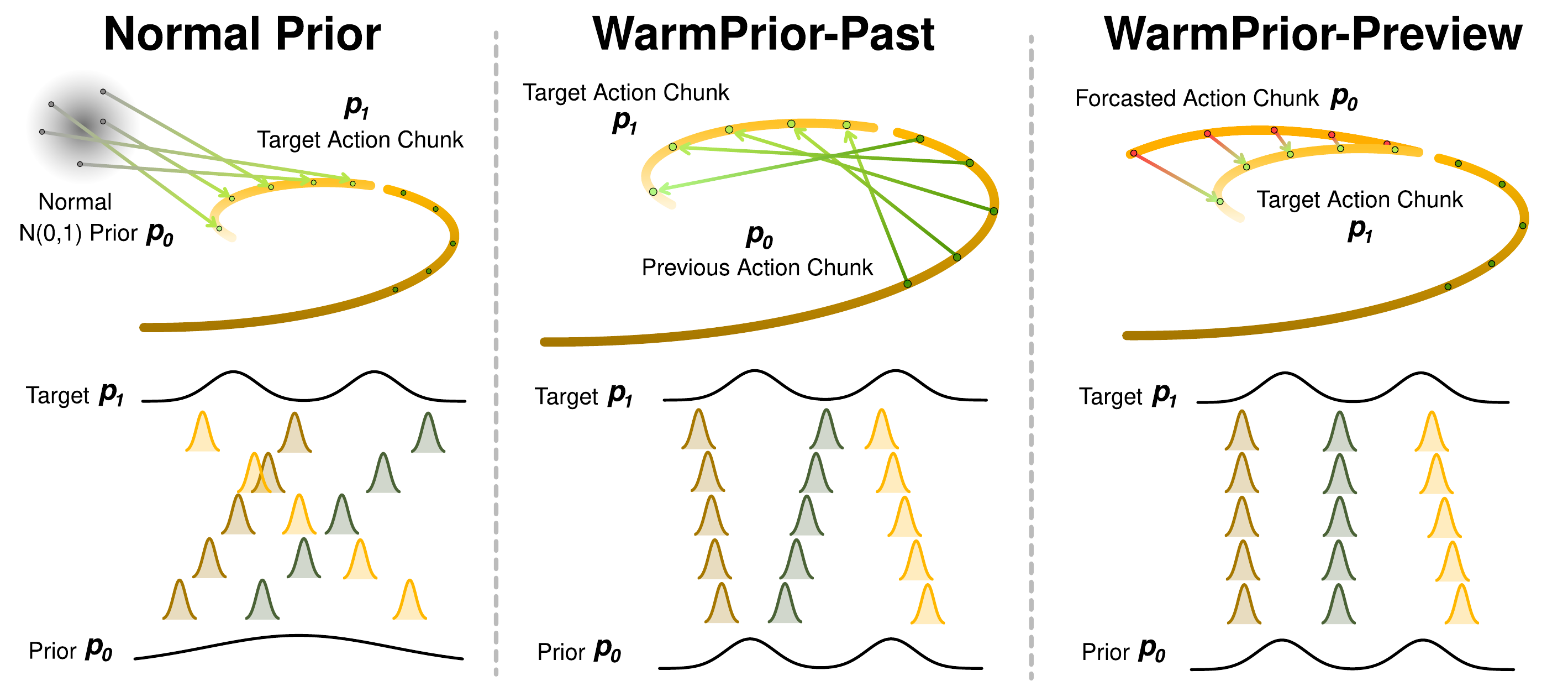}
  \caption{\textbf{WarmPrior.} Standard flow-matching policies transport samples from a context-free $\mathcal{N}(0,I)$ to the action manifold (left). WarmPrior initializes the transport from a temporally grounded Gaussian centered on the recent past-action chunk (Past) or on the model's own previous forecast of the current chunk (Preview) (middle, right). The resulting probability path is shorter, straighter, and temporally correlated across consecutive chunks.}
  \label{fig:main-idea}
\end{figure}

We introduce \textbf{WarmPrior}, which replaces this stateless source with a \emph{temporally grounded} prior whose mean is anchored on recent action history (\cref{fig:main-idea}).
We instantiate it in two minimal variants: \emph{WP-Past} anchors the prior on the previously executed action chunk, while \emph{WP-Preview} trains the policy to predict twice the chunk length at each inference step and reuses the model's own previous forecast of the current chunk as the prior mean.
Both add a residual Gaussian perturbation $\sigma\, \varepsilon$ so that the source remains a proper distribution, and both leave the network, the interpolant, and the integrator untouched (\cref{sec:method}).

This deliberately minimal intervention yields gains that compound along three independent axes.
\emph{Geometrically}, starting close to the target manifold shortens the transport and straightens the learned probability paths, acting as an implicit optimal-transport coupling that suppresses the irreducible endpoint ambiguity the network would otherwise average over (\cref{sec:straight}).
\emph{Temporally}, the residual scale $\sigma$ becomes a continuous knob between within-rollout commitment and multimodal expressiveness, supplying an implicit form of the consistency that action chunking enforces explicitly, and largely recovering baseline performance even when chunking is disabled (\cref{sec:ac1}).
\emph{Downstream}, WarmPrior recenters and shrinks the search space of prior-space reinforcement learning around a temporally grounded mean, so a tighter residual action on top of a pretrained policy outperforms vanilla DSRL~\cite{dsrl2025} in both sample efficiency and asymptotic performance (\cref{sec:rl}).

Empirically, on Robomimic, MimicGen, and a real Franka Research 3 setup, WarmPrior consistently improves success rate over the $\mathcal{N}(0,I)$ baseline with both the Diffusion Policy backbone~\cite{chi2023diffusionpolicy} and the VLA model GR00T N1.5~\cite{nvidia2025gr00tn15}; the improvement is largest at the lowest inference budgets and on the harder tasks, where the curvature of the flow matters most (\cref{sec:main-results}).
Taken together, these results promote the \textbf{\emph{source distribution}} from an inherited default to a first-class, and previously underexplored, design axis in generative robotic control.

\section{Background and Related Work}
\label{sec:background}

\paragraph{Flow-matching policies.}
Flow matching~\cite{lipman2023flow,albergo2023stochastic} trains a velocity network $v_\theta(t, a_t, o)$ along the linear interpolant $a_t = (1-t)\, a_0 + t\, a_1$ between a source $a_0 \sim p_0$ and data $a_1 \sim p_{\mathrm{data}}(\cdot \mid o)$, and samples by integrating $\dot a_t = v_\theta(t, a_t, o)$ from $a_0$. This paradigm underlies diffusion and flow-matching policies for behavior cloning~\cite{chi2023diffusionpolicy,janner2022planning,braun2024rfmp,hu2024adaflow,chisari2024pointflowmatch} and vision-language-action models~\cite{nvidia2025gr00tn1,black2024pi0,pi05}. Nearly all of them use $p_0 = \mathcal{N}(0, I)$; our work revisits that choice.

\paragraph{Optimal-transport couplings and straightened flows.}
Under the independent coupling $(a_0, a_1) \sim p_0 \otimes p_{\mathrm{data}}$, crossing trajectories force the velocity network to average over ambiguous endpoints, producing curved paths. Rectified Flow~\cite{liu2023rectified}, Multisample Flow Matching~\cite{pooladian2023multisample}, OT-CFM~\cite{tong2024conditional}, and Schr\"odinger-bridge variants~\cite{shi2023dsbm,tong2024sf2m} all reshape this \emph{coupling} to approximate dynamic OT. WarmPrior is complementary: it leaves the coupling independent and instead reshapes the \emph{source distribution} so the flow begins already close to data, straightening paths (\cref{sec:straight}) without any OT solver or retraining stage.

\paragraph{Informed priors for generative robot policies.}
Modifying the source of a generative policy is a small but emerging direction.
BRIDGER~\cite{chen2024bridger} replaces the Gaussian source with a data-aware, non-Gaussian source policy and bridges it to the expert distribution via stochastic interpolants.
In concurrent work, STEP~\cite{li2026step} trains an auxiliary action predictor whose output, perturbed by scheduled Gaussian noise, is injected at an \emph{intermediate} denoising step rather than at $t=0$, so the warm start lives inside the diffusion trajectory.
A2A~\cite{jia2026a2a} also anchors the prior on past actions, but encodes them deterministically into a latent source and composes deterministic ODE and decoder on top, making it effectively a history-conditioned \emph{deterministic flow transport model} rather than a stochastic generative sampler.
In contrast, WarmPrior preserves the stochastic flow-matching formulation end-to-end and focuses squarely on how to construct the \emph{prior space} $p_0$ itself (\cref{sec:method}).

\newcounter{algline}
\begin{algorithm}[!t]
  \caption{Training and inference of FM policy with WarmPrior.}
  \label{alg:warmprior}
  \vspace{0.2em}
  \small\linespread{1.0}\selectfont
  \begin{algorithmic}[1]
  \STATE {\bfseries Input:} dataset $\mathcal{D}$, interpolant $(\alpha, \beta)$, noise scale $\sigma$, chunk length $H$ (prediction length is $H$ for Past, $2H$ for Preview)
  \STATE {\bfseries Parameters:} velocity net $v_\theta$ (learnable)
  \makeatletter\setcounter{algline}{\value{ALC@line}}\makeatother
  \end{algorithmic}
  \vspace{0.0em}

  \noindent
  \begin{tabular}{@{}p{0.475\textwidth}@{\hspace{0.5em}}!{\color{black!25}\vrule}@{\hspace{0.5em}}p{0.475\textwidth}@{}}
  \textbf{\textsc{Training}} & \textbf{\textsc{Inference}} \\
  \begin{minipage}[t]{\linewidth}
    \vspace{-0.5em}
    \begin{algorithmic}[1]
    \makeatletter\setcounter{ALC@line}{\value{algline}}\makeatother
    \FOR{each iteration}
      \STATE Sample $(o, a_1, i) \sim \mathcal{D}$
      \STATE Draw $\varepsilon \sim \mathcal{N}(0, I)$ matching $a_1$; set $a_0 \gets \varepsilon$ \label{line:init}
      \IF{\textsc{Past}} \label{line:past}
        \STATE $a_0 \gets a^{\mathrm{data}}[i\!-\!H{:}i] + \sigma\, \varepsilon$ (when $i\!\geq\! H$)
      \ELSIF{\textsc{Preview}} \label{line:preview}
        \STATE $a_0[0{:}H] \gets a_1[0{:}H] + \sigma\, \varepsilon[0{:}H]$
      \ENDIF
      \STATE $t \sim \mathcal{U}(0,1)$
      \STATE $a_t \gets \alpha(t)\, a_0 + \beta(t)\, a_1$
      \STATE $\mathcal{L} \gets \| v_\theta(t, a_t, o) - (\dot\alpha\, a_0 + \dot\beta\, a_1) \|_2^2$
      \STATE Gradient step on $\theta$
    \ENDFOR
    \makeatletter\setcounter{algline}{\value{ALC@line}}\makeatother
    \end{algorithmic}
  \end{minipage}
  &
  \begin{minipage}[t]{\linewidth}
    \vspace{-0.5em}
    \begin{algorithmic}[1]
    \makeatletter\setcounter{ALC@line}{\value{algline}}\makeatother
    \STATE $\hat{a}^{\mathrm{prev}} \gets \varnothing$; reset env, observe $o$
    \WHILE{episode not done}
      \STATE Draw $\varepsilon \sim \mathcal{N}(0, I)$; set $a_0 \gets \varepsilon$
      \IF{\textsc{Past} \textbf{and} $\hat{a}^{\mathrm{prev}} \neq \varnothing$}
        \STATE $a_0 \gets \hat{a}^{\mathrm{prev}} + \sigma\, \varepsilon$
      \ELSIF{\textsc{Preview} \textbf{and} $\hat{a}^{\mathrm{prev}} \neq \varnothing$}
        \STATE $a_0[0{:}H] \gets \hat{a}^{\mathrm{prev}}[H{:}2H] + \sigma\, \varepsilon[0{:}H]$
      \ENDIF
      \STATE $\hat{a} \gets \textsc{FMSample}(v_\theta, a_0, o)$
      \STATE Execute $\hat{a}[0{:}H]$; observe next $o$
      \STATE $\hat{a}^{\mathrm{prev}} \gets \hat{a}$
    \ENDWHILE
    \end{algorithmic}
  \end{minipage}
  \end{tabular}
  \vspace{0.2em}
\end{algorithm}

\section{WarmPrior}
\label{sec:method}

WarmPrior modifies only the source distribution of a flow-matching policy: it reshapes $p_0$ while leaving the network, interpolant, and training objective untouched.
We instantiate it as two minimal variants, \emph{WarmPrior-Past} (WP-Past) and \emph{WarmPrior-Preview} (WP-Preview), which differ only in how the prior mean is anchored to the agent's own action history.
Below we formalize the common template (\cref{alg:warmprior}) and then specify each variant in turn.

\paragraph{Formulation.}
Let $a_0$ denote the sample drawn from the prior that the flow-matching ODE transports into the predicted action chunk, with shape $H \times d_a$ for Past and $2H \times d_a$ for Preview.
For a warm index set $\mathcal{W}$ over the prediction positions, with cold complement $\mathcal{C}$, and a mean $\mu$ defined on $\mathcal{W}$, WarmPrior samples
\begin{equation}
  a_0[\tau] \;=\;
  \begin{cases}
    \mu_\tau + \sigma\, \varepsilon_\tau, & \tau \in \mathcal{W}, \\
    \varepsilon_\tau, & \tau \in \mathcal{C},
  \end{cases}
  \qquad \varepsilon \sim \mathcal{N}(0, I).
  \label{eq:warm-prior}
\end{equation}
The cold region keeps the vanilla flow-matching prior intact, so positions without a reliable anchor behave exactly as in the standard flow-matching baseline.
The scalar $\sigma > 0$ controls the residual noise on warm positions so that the warm region remains a proper distribution rather than a deterministic point mass; we fix $\sigma$ per variant below and revisit it as a multimodality knob in \cref{sec:ac1}.
Under this formulation, WarmPrior is fully specified by the pair $(\mathcal{W}, \mu)$ together with the prediction length. Our primary goal is to start the generative flow from a \textbf{\emph{plausible target action}} rather than pure noise, and we propose two variants that differ in how the prior mean $\mu$ is anchored.

\paragraph{WarmPrior-Past.}
\label{sec:method:past}
The simplest plausible target is the previous action chunk: WP-Past predicts a single chunk of $H$ actions and anchors $\mu$ on the previous action chunk.

At training, for each sample with in-buffer index $i$, we retrieve the $H$ preceding actions $a^{\mathrm{data}}_{i-H:i}$ from the replay buffer (normalized to the training action space), verify via a binary search on episode boundaries that the window lies within a single episode, and set:
\begin{equation}
  \mu^{\mathrm{Past}}_\tau \;=\; a^{\mathrm{data}}[i-H+\tau], \qquad \text{for } \tau \in \{0, \dots, H-1\}.
  \label{eq:past-train}
\end{equation}
When the window would cross an episode boundary (e.g., at the start of a demonstration), the sample falls back to $\mathcal{W} = \emptyset$.

At inference, we directly use the previously executed action chunk, setting $\mu^{\mathrm{Past}}_\tau = \hat{a}^{\mathrm{prev}}_\tau$ with $\mathcal{W} = \{0, \dots, H-1\}$, and fall back to $\mathcal{W} = \emptyset$ at the first chunk. We use $\sigma = 0.5$ for this variant.

\paragraph{WarmPrior-Preview.}
\label{sec:method:preview}
WP-Preview trains the policy to look one chunk further than it needs to: instead of predicting a single chunk of $H$ actions, it predicts $2H$ actions at each inference step and executes only the first $H$.
The second $H$ steps serve as a \emph{preview} of the next chunk, acting as the model's own forecast of future actions.
When the next decision step arrives, this preview aligns exactly with the first $H$ positions of the new prediction, providing a natural and highly accurate prior mean for the next generation process.
Crucially, across both training and inference, the $2H$-step generation is strictly partitioned: the first $H$ steps (the actions to be executed) are generated starting from the WarmPrior ($\mathcal{W} = \{0, \dots, H-1\}$), while the second $H$ steps (the preview) are generated starting from pure Gaussian noise ($\mathcal{C} = \{H, \dots, 2H-1\}$).

At training, we face a chicken-and-egg problem: the ideal prior mean would be the model's own past preview, which is unavailable before the model is trained.
However, the ground-truth target itself is precisely the limit that a perfectly calibrated preview would converge to: at convergence, the model's prior forecast of the current chunk should coincide with the chunk itself.
Thus, we can simply use the ground-truth target itself as a proxy for a perfectly calibrated preview:
\begin{equation}
  \mu^{\mathrm{Preview}}_\tau \;=\; a_1[\tau], \qquad \text{for } \tau \in \{0, \dots, H-1\},
  \label{eq:preview-train}
\end{equation}
where $a_1 \in \mathbb{R}^{2H \times d_a}$ spans the full $2H$-step horizon. We use $\sigma = 1.0$ for this variant.

At inference, let $\hat{a}^{\mathrm{prev}} \in \mathbb{R}^{2H \times d_a}$ be the previous $2H$-step prediction.
WP-Preview sets
\begin{equation}
  \mu^{\mathrm{Preview}}_\tau \;=\; \hat{a}^{\mathrm{prev}}[H+\tau], \qquad \text{for } \tau \in \{0, \dots, H-1\},
  \label{eq:preview-infer}
\end{equation}
so that the warm first half of the new prior carries the previous forecast of the current chunk, while the cold second half covers the new horizon that no previous preview has seen.
At the first chunk of an episode, where no previous prediction exists, we fall back to $\mathcal{W} = \emptyset$.

\begin{table}[!t]
  \caption{\textbf{Simulation success rate (\%)} on Robomimic and MimicGen (image) at $H=8$ across three inference budgets. Parentheses show the absolute gain over the $\mathcal{N}(0,I)$ baseline; \textcolor{goodgreen}{green} marks gains exceeding $\sigma_{\mathrm{base}} + \sigma_{\mathrm{method}}$ (non-overlapping $1\sigma$ seed intervals). Best per (task, NFE) in \textbf{bold}.}
  \label{tab:main-sr}
  \centering
  \begin{footnotesize}
  \setlength{\tabcolsep}{3.5pt}
  \begin{tabular}{l@{\hskip 6pt}ccc@{\hskip 8pt}ccc@{\hskip 8pt}ccc}
    \toprule
      & \multicolumn{3}{c}{\textbf{NFE $=9$}} & \multicolumn{3}{c}{\textbf{NFE $=3$}} & \multicolumn{3}{c}{\textbf{NFE $=1$}} \\
    \cmidrule(lr){2-4}\cmidrule(lr){5-7}\cmidrule(lr){8-10}
    Task & Base & WP-Past & WP-Preview & Base & WP-Past & WP-Preview & Base & WP-Past & WP-Preview \\
    \midrule
    \multicolumn{10}{l}{\textit{Robomimic — state observation}} \\
    Square-PH        & 86.7 & \textbf{88.1} \dpos{1.4}   & \textbf{88.1} \dgood{1.4}  & 86.2 & \textbf{88.0} \dgood{1.8}  & 87.9 \dgood{1.7}           & 83.6 & 86.6 \dgood{3.0}            & \textbf{87.3} \dgood{3.7}  \\
    Square-MH        & 65.9 & 69.2 \dgood{3.3}           & \textbf{72.7} \dgood{6.8}  & 65.4 & \textbf{73.2} \dgood{7.8}  & 72.9 \dgood{7.5}           & 65.9 & 70.1 \dpos{4.2}             & \textbf{77.8} \dgood{11.9} \\
    Transport-PH     & 34.1 & 36.2 \dpos{2.1}            & \textbf{43.3} \dgood{9.2}  & 39.0 & 44.0 \dgood{5.0}           & \textbf{49.1} \dgood{10.1} & 36.8 & 39.8 \dpos{3.0}             & \textbf{47.6} \dgood{10.8} \\
    Transport-MH     & 16.3 & 20.7 \dgood{4.4}           & \textbf{24.3} \dgood{8.0}  & 21.3 & \textbf{30.7} \dgood{9.4}  & 30.4 \dgood{9.1}           & 23.3 & 30.2 \dgood{6.9}            & \textbf{34.5} \dgood{11.2} \\
    Tool-Hang-PH     & 79.4 & 80.6 \dpos{1.2}            & \textbf{82.8} \dgood{3.4}  & 72.3 & 75.1 \dgood{2.8}           & \textbf{75.8} \dpos{3.5}   & 77.7 & 78.2 \dpos{0.5}             & \textbf{81.9} \dgood{4.2}  \\
    \midrule
    \multicolumn{10}{l}{\textit{Robomimic — image observation}} \\
    Square-PH        & 86.9 & 88.2 \dpos{1.3}            & \textbf{88.7} \dgood{1.7}  & 87.7 & 89.2 \dgood{1.4}           & \textbf{89.6} \dgood{1.9}  & 88.7 & \textbf{89.3} \dpos{0.6}    & 89.1 \dpos{0.4}            \\
    Square-MH        & 76.1 & \textbf{78.0} \dpos{1.9}   & 77.8 \dpos{1.7}            & 73.8 & \textbf{77.9} \dgood{4.1}  & 77.1 \dgood{3.2}           & 72.4 & \textbf{77.6} \dgood{5.2}   & 75.1 \dgood{2.7}           \\
    Transport-PH     & 92.8 & \textbf{94.5} \dgood{1.7}  & 94.3 \dgood{1.6}           & 92.1 & 93.9 \dpos{1.9}            & \textbf{94.9} \dgood{2.9}  & 91.3 & 93.4 \dgood{2.2}            & \textbf{93.7} \dgood{2.4}  \\
    Transport-MH     & 74.8 & 79.7 \dgood{4.9}           & \textbf{79.8} \dgood{4.9}  & 73.8 & 80.0 \dgood{6.2}           & \textbf{80.7} \dgood{6.9}  & 74.3 & 78.6 \dpos{4.3}             & \textbf{79.7} \dpos{5.4}   \\
    Tool-Hang-PH     & 43.7 & 45.8 \dpos{2.1}            & \textbf{56.3} \dgood{12.6} & 36.9 & 38.4 \dpos{1.4}            & \textbf{50.7} \dgood{13.8} & 41.3 & 38.9 \dneg{2.4}             & \textbf{54.0} \dgood{12.7} \\
    \midrule
    \multicolumn{10}{l}{\textit{MimicGen — image observation}} \\
    Stack            & 21.4 & 22.8 \dpos{1.4}            & \textbf{31.6} \dgood{10.2} & 21.3 & 23.7 \dpos{2.4}            & \textbf{30.7} \dgood{9.4}  & 21.3 & 22.4 \dpos{1.1}             & \textbf{28.7} \dgood{7.4}  \\
    Coffee           & 26.8 & 29.6 \dpos{2.8}            & \textbf{34.7} \dgood{7.9}  & 23.3 & 24.1 \dpos{0.8}            & \textbf{33.4} \dgood{10.1} & 16.2 & 20.4 \dpos{4.2}             & \textbf{29.4} \dgood{13.2} \\
    Threading        & 13.8 & 15.5 \dpos{1.7}            & \textbf{20.9} \dgood{7.1}  & 16.3 & 16.6 \dpos{0.3}            & \textbf{22.0} \dgood{5.7}  & 12.5 & 15.6 \dpos{3.1}             & \textbf{18.0} \dgood{5.5}  \\
    \bottomrule
  \end{tabular}
  \end{footnotesize}
\end{table}

\emph{Optimal transport.} Among the WarmPrior variants we consider, Preview is the choice that pushes the prior mean as close as possible to the target: when the preview is accurate, the flow starts directly on the model's own forecast of the current chunk and only has to correct its residual error (\cref{sec:straight}).

\emph{Residual policy interpretation.} Because the warm portion of the prior already \emph{is} a prediction of the current chunk, the flow only needs to learn the correction $a_1 - \mu^{\mathrm{Preview}}$ on top of a committed forecast. In this sense, WP-Preview turns the generative policy into a residual policy that refines its own previous plan.

\section{Main Results}
\label{sec:main-results}

\subsection{Simulation}
\label{sec:main-sim}

\paragraph{Setup.}
We evaluate in simulation on two robotic manipulation benchmarks: Robomimic~\cite{robomimic2021} and MimicGen~\cite{mandlekar2023mimicgen}.
On Robomimic we evaluate under both state- and image-observation regimes on \textsc{Square}, \textsc{Transport}, and \textsc{Tool-Hang} in the \textsc{PH} (proficient-human) split, plus the harder \textsc{MH} (multi-human) splits for \textsc{Square} and \textsc{Transport}, omitting \textsc{Lift} and \textsc{Can} on which the flow-matching policy already saturates near 100\% success rate.
On MimicGen we use the human-demonstration datasets (10 demos per task) for \textsc{Stack}, \textsc{Coffee}, and \textsc{Threading} under image observations.

We evaluate WarmPrior on the Diffusion Policy (ChiTransformer)~\cite{chi2023diffusionpolicy}, a widely adopted policy architecture, trained here with flow matching.
All methods share the linear flow interpolant; only the source distribution changes.
Since behavior-cloning training curves for Diffusion Policy on Robomimic are known to be noticeably noisy across checkpoints~\cite{robomimic2021}, we train these models for a sufficient 200k iterations at a batch size of 1024 (state) or 256 (image). To mitigate this variance, we evaluate at regularly spaced checkpoints and average the performance of the top-3 checkpoints per seed. The success rate is computed over 200 episodes and 3 seeds at three inference budgets (NFE $\in \{9,3,1\}$).
Unless stated otherwise, the action-chunk length is $H = 8$ for both Robomimic and MimicGen.
Full training hyperparameters and additional implementation details are provided in~\cref{app:training}.

\paragraph{Performance improvements.}
\cref{tab:main-sr} reports the full NFE sweep.
The majority of evaluations exhibit non-overlapping one-standard-deviation intervals between the baseline and our method (green deltas), demonstrating that this simple modification to the prior distribution yields a \textit{highly significant} performance boost. Furthermore, bold values highlight the best performance among the evaluated methods. While WP-Past achieves respectable performance gains, WP-Preview demonstrates even greater improvements. Finally, we observe that the magnitude of these improvements is most pronounced at the lowest inference budget, with the \textit{largest} average performance gains occurring at NFE $= 1$. We discuss the underlying reasons for both observations in \cref{sec:straight}.

\subsection{Real-Robot Experiments}
\label{sec:main-real}

To validate our approach in the real world, we deploy our method on a Franka Research 3. As illustrated in \cref{fig:real-tasks}, we construct four tabletop manipulation tasks and collect human teleoperation demonstrations using the DROID platform setup~\cite{khazatsky2024droid}. Each task is trained on a dataset of 30 demonstrations.

For the policy architecture, we employ the GR00T N1.5 VLA model~\cite{nvidia2025gr00tn15}, which also utilizes a flow-matching action head. The models are trained for 30k steps with a batch size of 64; further training details are provided in~\cref{app:training}. During inference, the number of function evaluations (NFE) is fixed to 4. We evaluate the performance using 3 independent training seeds, conducting 50 evaluation trials per seed. As reported in \cref{fig:real-bar}, WarmPrior consistently improves the overall success rate across all four real-world tasks, with the largest gains on the precision-demanding \emph{Cable Insertion} and \emph{Block Stacking}.

\begin{figure}[t]
  \centering
  \begin{minipage}[b]{0.36\linewidth}
    \centering
    \includegraphics[width=\linewidth]{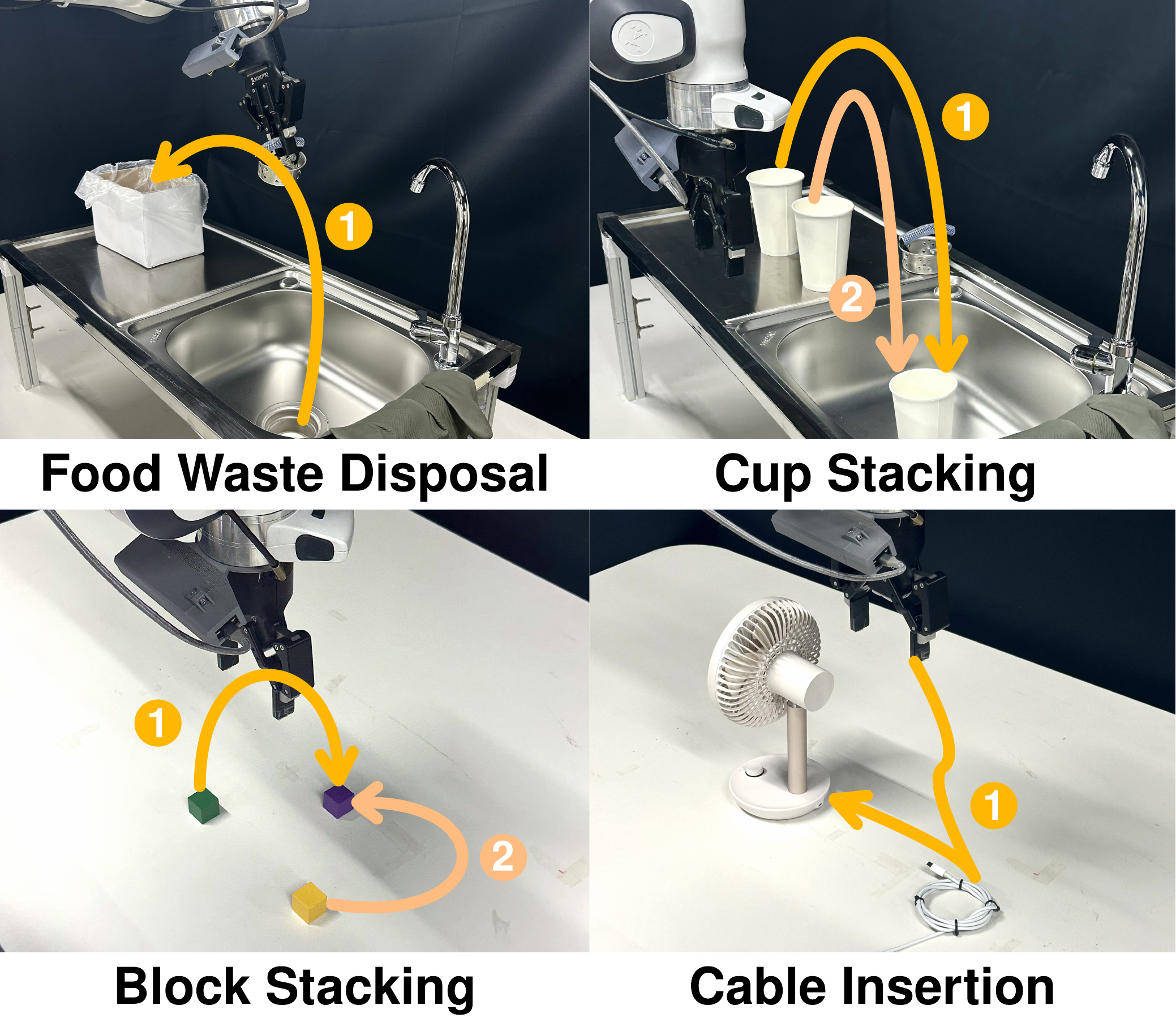}
    \caption{\textbf{Real-robot tasks.} Four tabletop manipulation scenes used in \cref{fig:real-bar}: \emph{Food Waste Disposal}, \emph{Cup Stacking}, \emph{Block Stacking}, and \emph{Cable Insertion}.}
    \label{fig:real-tasks}
  \end{minipage}%
  \hspace{1.5em}%
  \begin{minipage}[b]{0.55\linewidth}
    \centering
    \includegraphics[width=\linewidth]{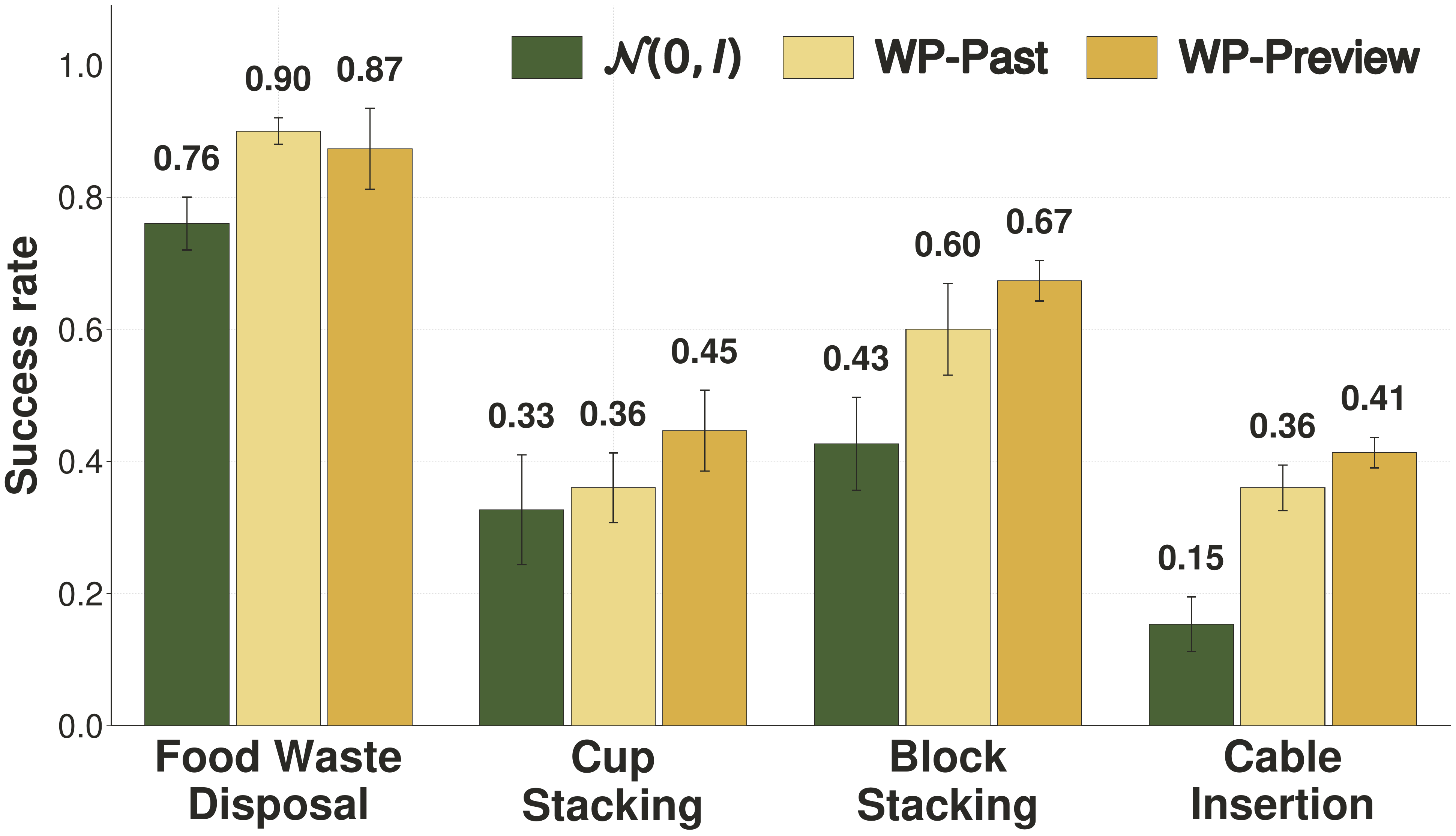}
    \caption{\textbf{Real-robot success rate.} We evaluate WP-Past, WP-Preview, and the $\mathcal{N}(0,I)$ baseline on four tabletop manipulation tasks, reporting the mean and standard deviation over three training seeds (50 trials per seed).}
    \label{fig:real-bar}
  \end{minipage}
\end{figure}

\section{Understanding and Extending WarmPrior}
\label{sec:analysis}

In this section, we investigate \emph{why} replacing the standard $\mathcal{N}(0,I)$ source with WarmPrior translates into the consistent gains of \cref{sec:main-results}, and what further consequences follow.
\cref{sec:straight} gives a geometric account: WarmPrior shortens the transport and straightens the learned probability paths.
\cref{sec:ac1} reveals a second, independent benefit, \emph{temporal consistency}: WarmPrior supplies a $\sigma$-tunable form of the consistency that action chunking provides explicitly, and the effect is most pronounced when explicit chunking is turned off ($H=1$).
\cref{sec:rl} then extends the same prior to reinforcement learning, showing that it also reshapes the exploration space of prior-space RL.
The first two subsections explain the behavior-cloning gains of \cref{tab:main-sr}; the third is a natural extension of WarmPrior to a downstream setting.

\subsection{WarmPrior Improves SR by Straightening Flow Trajectories}
\label{sec:straight}

\begin{wrapfigure}{r}{0.45\linewidth}
  \vspace{-1em}
  \centering
  \includegraphics[width=\linewidth]{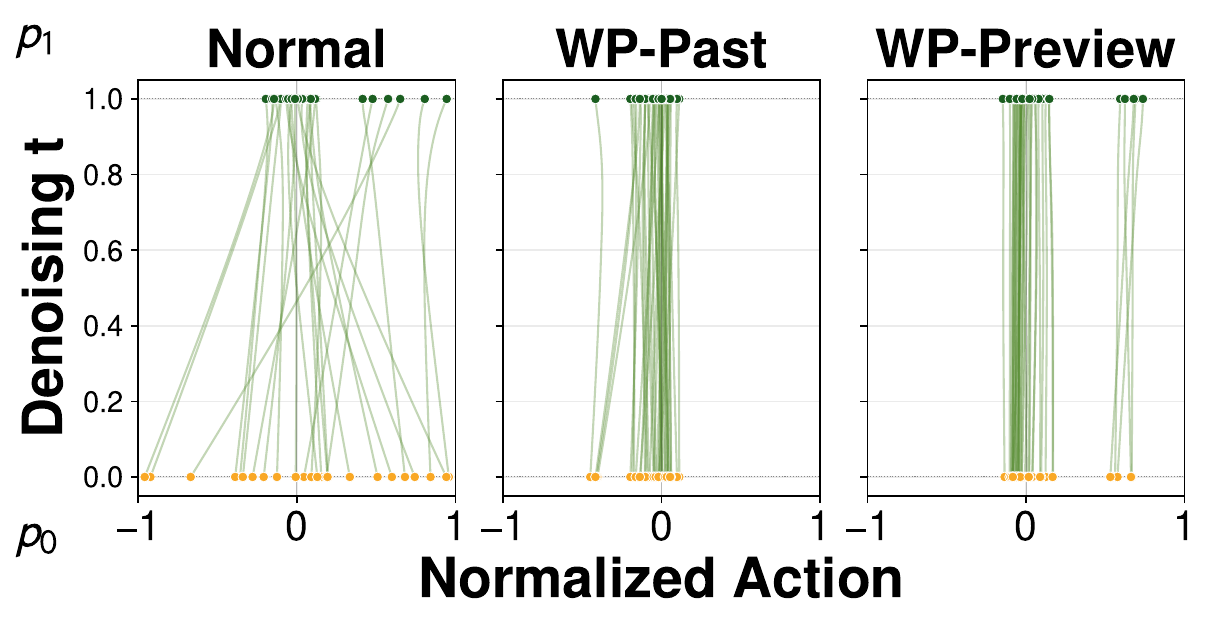}
  \caption{\textbf{Flow trajectories on \textsc{Square-MH}.} Normalized action coordinate vs.\ denoising time $t\!\in\![0,1]$; bottom markers: prior $p_0$, top markers: prediction $p_1$.}
  \label{fig:straight}
\end{wrapfigure}

\paragraph{Empirical observation.}
\cref{fig:straight} shows the integration paths of $a_t$ for a baseline flow policy and for WarmPrior on the same observations from Robomimic \textsc{Square-MH}.
The baseline paths curve noticeably as the network pulls samples from a random origin onto the action manifold; the WarmPrior paths, already starting close to the manifold, are visibly straighter and more parallel.
Intuitively, because fewer flows cross one another, the conditional flow-matching network spends less capacity realigning samples from the random base distribution and can devote more to refining actions, exactly where it matters for downstream success rate.

\begin{wraptable}{r}{0.45\linewidth}
  \vspace{-1em}
  \centering
  \caption{Pathwise curvature $\kappa(o)$ of the learned flow on state-observation tasks (\cref{eq:curvature}; lower is straighter; values normalized so the $\mathcal{N}(0,I)$ baseline reads $1.000$).}
  \label{tab:curvature}
  \vspace{0.3em}
  \footnotesize
  \setlength{\tabcolsep}{3pt}
  \begin{tabular}{lccc}
    \toprule
    Task & $\mathcal{N}(0,I)$ & WP-Past & WP-Preview \\
    \midrule
    Square-PH     & 1.000 & 0.823 & 0.803 \\
    Square-MH     & 1.000 & 0.705 & 0.559 \\
    Transport-PH  & 1.000 & 0.720 & 0.692 \\
    Transport-MH  & 1.000 & 0.695 & 0.637 \\
    Tool-Hang-PH  & 1.000 & 0.806 & 0.807 \\
    \bottomrule
  \end{tabular}
\end{wraptable}

\paragraph{Curvature diagnostic.}
To make the observation quantitative, we measure the pathwise curvature of the learned flow.
For a smooth path $a: [0,1] \to \mathbb{R}^{H \times d_a}$ with $\dot a_t = v_\theta(t, a_t, o)$ we use the standard velocity-variance surrogate
\begin{equation}
  \kappa(o) \;=\; \int_0^1 \| \dot a_t - \bar v \|_2^2 \, dt,
  \qquad
  \bar v = \int_0^1 \dot a_t \, dt,
  \label{eq:curvature}
\end{equation}
evaluated by finite differences along the Euler sampler with $N=100$ steps.
We compute $\kappa(o)$ over $2{,}000$ validation observations and report the average in \cref{tab:curvature}.
Every task exhibits a reduction in mean curvature, and the relative reduction tracks the success-rate gain of \cref{tab:main-sr}: tasks with the largest curvature reduction (\textsc{Square-MH}, \textsc{Transport-MH}) are also the tasks with the largest SR gain, supporting the straightening-explains-performance hypothesis.

\paragraph{Branching cost: an irreducible residual.}
The curvature reduction has a measure-theoretic origin we call the \emph{branching cost}.
Vectorize an action chunk into \(\mathbb R^d\), let \((A_0,A_1)\sim\Pi_o\) denote the conditional joint law of source and target, and write \(A_t=(1\!-\!t)A_0+tA_1\).
The flow-matching objective \(\mathcal L_o(v)=\int_0^1\mathbb E_{\Pi_o}[\|v_t(A_t,o)-(A_1-A_0)\|^2\mid o]\,dt\) regresses the transport direction \(A_1-A_0\), and because only \((A_t,o)\) is observable, the best attainable predictor is the conditional expectation \(v_t^\star(x,o)=\mathbb E[A_1-A_0\mid A_t=x,o]\).
The residual error this predictor cannot eliminate,
\begin{equation}
  \mathcal B(o)
  \;\coloneqq\;
  \mathcal L_o(v^\star)
  \;=\;
  \int_0^1
  \frac{\mathbb E\!\bigl[\|A_1-\mathbb E[A_1\!\mid\! A_t,o]\|^2\,\bigm|\,o\bigr]}{(1-t)^2}
  \,dt ,
  \label{eq:branching-cost}
\end{equation}
measures how ambiguous \(A_1\) remains after observing \(A_t\): when many distinct targets share an \(A_t\), \(v^\star\) must average over them and the trajectory bends.
A standard total-variance decomposition splits the coupling cost \(\mathbb E[\|A_1-A_0\|^2\mid o]\) into the kinetic action of \(v^\star\) plus \(\mathcal B(o)\) (see \cref{app:ot_detailed} for the full derivation); the second term is pure excess caused by directional ambiguity and vanishes for OT couplings, where \(\mathcal B\equiv 0\)~\cite{mccann1997,benamou2000}.

\paragraph{How WarmPrior reduces the branching cost.}
WarmPrior writes the source as \(A_0=P_{\mathcal W}(\mu+\sigma\Xi)+P_{\mathcal C}\Xi\) with \(\Xi\sim\mathcal N(0,I)\) independent of \((A_1,\mu)\) given \(o\), where \(P_{\mathcal W}\) projects onto the warm coordinates (\(d_{\mathcal W}\) dimensions) and \(P_{\mathcal C}=I-P_{\mathcal W}\).
Bounding the optimal predictor's error by that of the simpler predictor \(P_{\mathcal W}A_t\) cancels the \((1-t)^2\) factor in \eqref{eq:branching-cost} (\cref{app:ot_detailed}, Proposition~\ref{prop:app-warm-bound}), giving
\begin{equation}
  \mathcal B_{\mathcal W}(o)
  \;\le\;
  \underbrace{\mathbb E\!\left[\|P_{\mathcal W}(A_1\!-\!\mu)\|^2\,\middle|\,o\right]}_{\text{mean mismatch}}
  +\;\sigma^2 d_{\mathcal W}.
  \label{eq:warmprior-branching-bound}
\end{equation}
The warm-coordinate branching cost is therefore controlled by two intuitive quantities: how well the prior mean \(\mu\) predicts the target, and how much residual noise \(\sigma\) is injected.
This immediately explains the ordering of our variants.
\textbf{Preview} sets \(\mu\) to a forecast of the current chunk, so the mismatch reduces to the forecast error \(\mathbb E[\|E\|^2\mid o]\) with \(P_{\mathcal W}A_1 = P_{\mathcal W}\mu+E\); in the idealized limit of an exact forecast (\(E=0\)) only the irreducible \(\sigma^2 d_{\mathcal W}\) term survives.
\textbf{Past} reuses the previously executed chunk, replacing \(E\) with the persistence residual \(R\) between consecutive chunks and yielding the same form \((\text{prediction error})+\sigma^2 d_{\mathcal W}\).
Whenever the forecaster improves on persistence (\(\mathbb E[\|E\|^2\mid o]\le\mathbb E[\|R\|^2\mid o]\)), Preview attains a tighter bound, matching the ordering observed in \cref{tab:main-sr}.
In both cases WarmPrior acts as an amortized approximation to the OT coupling, shortening transport and suppressing the directional ambiguity that bends the learned field.

The bound also exposes a trade-off along \(\sigma\), an axis separate from aligning \(\mu\): smaller \(\sigma\) tightens \(\sigma^2 d_{\mathcal W}\) in \cref{eq:warmprior-branching-bound} and favors a straighter field, but concentrates the source onto \(\mu\), which only helps if \(\mu\) is reliable.
In practice it is not (WP-Past carries the persistence residual, WP-Preview the forecast error), so an overly small \(\sigma\) leaves no slack to absorb this variability and degrades success rate.
The right \(\sigma\) balances straightness against robustness to prior-mean diversity; we defer the full ablation to \cref{app:sigma-ablation}.

\begin{figure}[t]
  \centering
  \includegraphics[width=\textwidth]{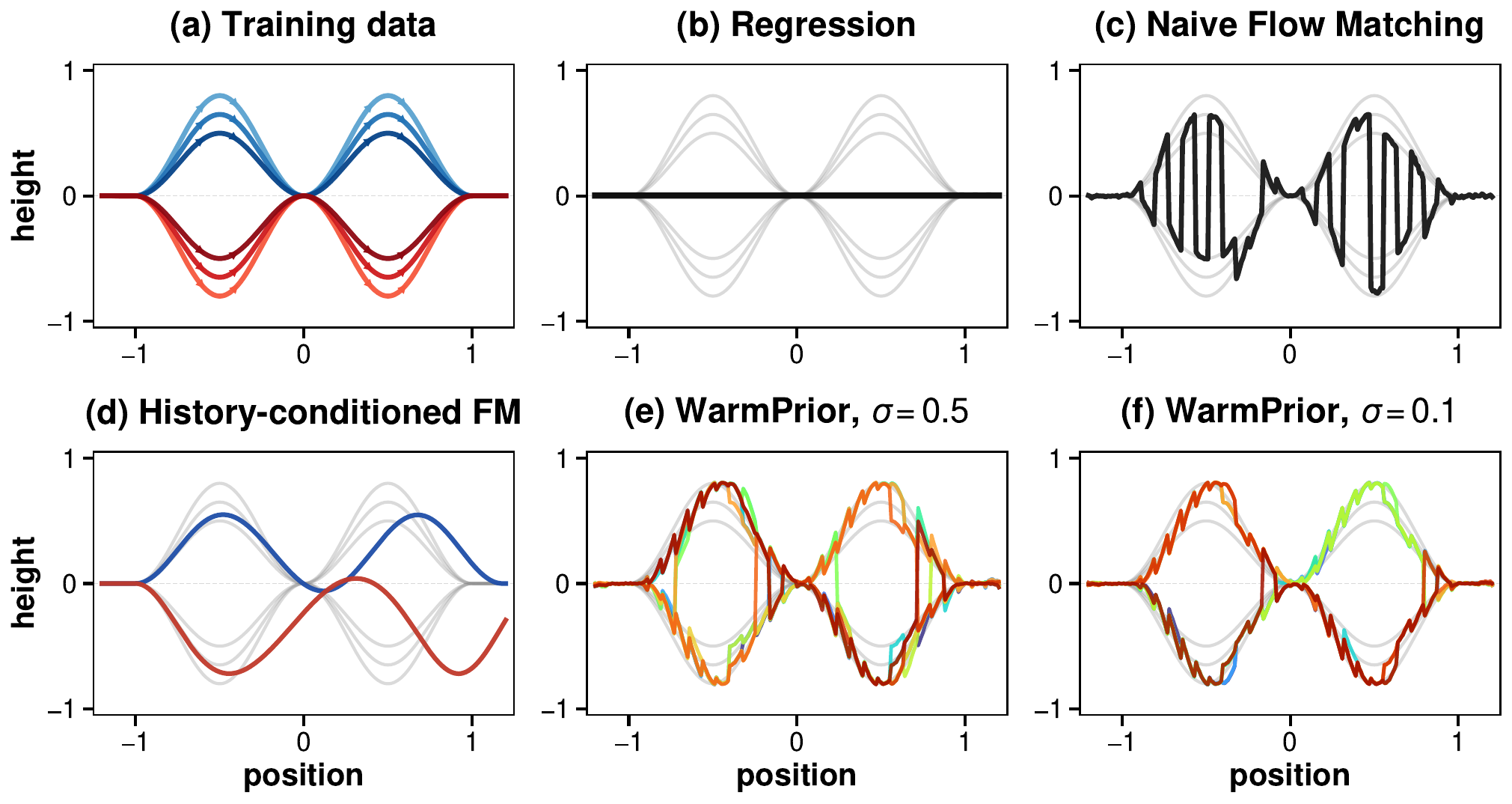}
  \caption{\textbf{Mode switching in a 1D navigation toy.}
  All policies share a 1024-d 4-layer MLP backbone trained for 50k iterations with batch size 256.
  Six demonstrations pass through two obstacles (three above, three below), inducing a bimodal $p(a\mid o)$ at each position.
  \textbf{(a)} training data; \textbf{(b)} regression collapses to the mean; \textbf{(c)} naive flow matching recovers both modes but oscillates between them; \textbf{(d)} history-conditioned flow matching commits within a rollout but drifts off-manifold under inference-time history shift; \textbf{(e, f)} WarmPrior at $\sigma\!=\!0.1$ and $\sigma\!=\!0.5$ commits per rollout, with $\sigma$ tuning between temporal consistency and multimodality.}
  \label{fig:toy}
\end{figure}

\subsection{WarmPrior as a Tunable Source of Temporal Consistency}
\label{sec:ac1}

Beyond the geometric benefit of \cref{sec:straight}, WarmPrior provides a second, independent advantage: a \emph{tunable} form of \emph{implicit temporal consistency} between consecutive inferences.

\paragraph{Mode switching in generative control policies.}
Consider the 1D navigation toy in \cref{fig:toy}(a), where the observation $o$ is the agent's horizontal position and the action $a$ is its vertical height. Demonstrations split evenly between passing above and passing below each obstacle, so the conditional distribution $p(a \mid o)$ is multimodal at every $o$.
A regression policy averages the branches and collapses to the mean, driving straight through the obstacle (\cref{fig:toy}, panel b).
A flow-matching policy trained on the standard objective recovers both modes, but only at the level of \emph{per-inference marginals}: the objective places no constraint linking the chunk produced at one inference to the chunk produced at the next.
The policy is therefore free to pick a different mode at each inference, yielding an execution that oscillates rapidly between them, a pathology we term \textbf{\emph{mode switching}} (\cref{fig:toy}, panel c).
Action chunking~\cite{zhao2023aloha} enforces commitment \emph{within} a chunk, but the objective still treats consecutive chunks independently and the oscillation persists at every chunk boundary.
A natural remedy is to condition the policy on the \emph{action history} $h$, but naive history conditioning is costly and fragile: it substantially slows convergence and inflates per-step compute and memory~\cite{koo2026hamlet}, and has two further drawbacks. First, conditioning on $h$ pins the policy to whichever mode its history already commits to, reducing the effective multimodality of $p(a \mid o)$: in \cref{fig:toy}(d), every rollout follows either an above-above or a below-below path, with no recombination across branches. Second, at inference time small execution errors compound into a distributional shift over $h$.

\paragraph{Tunable temporal consistency via $\sigma$.}
Because the prior mean is correlated with the previous action chunk, a small $\sigma$ keeps the new chunk inside the nearby mode's basin and prevents the flow from crossing between distant modes (\cref{fig:toy}, panel e), while a large $\sigma$ broadens the source and recovers more of the multimodal distribution at every step (\cref{fig:toy}, panel f).
The prior variance therefore acts as a continuous \emph{regulator} between \emph{temporal consistency} and \emph{multimodality}.
Crucially, unlike history conditioning or long action chunks that \emph{explicitly} enforce temporal consistency at training or inference time, WarmPrior only \emph{implicitly} biases the source distribution: within each rollout it commits the policy to a single coherent mode, while leaving the generative policy's room for multimodality intact across rollouts.

\paragraph{Isolating the consistency effect at $H=1$.}
To strip away explicit chunking and isolate the prior's implicit consistency bias, we set $H=1$, so the policy runs a fresh inference every timestep and the WarmPrior becomes the sole source of inter-step consistency.
\cref{fig:ac1-bar} reports SR at $H=1$ and NFE $=1$.
The $\mathcal{N}(0,I)$ baseline degrades sharply (e.g.\ \textsc{Transport-MH}: $23.3\%\!\to\!1.3\%$), while WarmPrior recovers most of the lost performance, with gains of up to $+14.8$ on MimicGen \textsc{Coffee}.
These gains are consistently larger than at the default $H=8$, confirming that the prior carries more weight once explicit chunking is stripped away.

\paragraph{Practical implication.}
Action chunking locks the policy to an $H$-step plan and cannot react to new observations within a chunk, which is a liability on tasks requiring fast reactivity.
WarmPrior offers an alternative that \emph{preserves temporal consistency while allowing per-step re-planning}, and we see this direction as a promising avenue for future work.

\subsection{WarmPrior Improves Prior-Space RL Efficiency}
\label{sec:rl}

Beyond behavior cloning, the same source-distribution shaping extends to the RL stage: using the WarmPrior-pretrained policy as the frozen base, the prior mean additionally reshapes the exploration space of prior-space reinforcement learning and yields a substantial efficiency gain.

\paragraph{Background: DSRL.}
DSRL~\cite{dsrl2025} fine-tunes a pretrained diffusion policy with reinforcement learning by acting in the prior space: instead of having the RL agent output actions directly, the agent proposes the prior sample $a_0$, which is then mapped deterministically to an action $a_1$ via the frozen pretrained policy's ODE sampler (DDIM~\cite{song2021ddim} or flow-matching).
The RL action space is $\mathbb{R}^{H \times d_a}$ with the shape of the prior.
By acting on the prior sample, DSRL eliminates the need to backpropagate through the diffusion sampler and makes the policy compatible with off-the-shelf RL algorithms.
Two algorithms are commonly used within this framework: DSRL-SAC, which applies SAC~\cite{haarnoja2018sac} directly to the noise-space MDP, and DSRL-NA, which exploits the diffusion policy's noise-aliasing structure through a dual-critic scheme that distills an action-space critic $Q^A$ into a noise-space critic $Q^W$.
However, the exploration space remains the uninformative $\mathcal{N}(0,I)$ prior, forcing the RL agent to search across the full prior from scratch.

\paragraph{Method: Conditioned-residual WarmPrior.}
WarmPrior offers an immediate structural improvement: because the WarmPrior mean is already close to the target action manifold, the RL agent only has to learn a bounded \emph{residual} around it.
Concretely, we extend the observation to $\tilde{o} = [o, \mu]$ and bound the RL action to a small magnitude $\delta$:
\begin{equation}
  a_0 = \mu + \Delta, \quad \Delta = \pi_{\mathrm{RL}}(\tilde o) \in [-\delta, \delta]^{H \times d_a}.
  \label{eq:option5}
\end{equation}
In practice we set $\delta = 0.5$, compared to $\delta = 1.5$ used by vanilla DSRL.
The RL agent now explores a $3\times$ tighter region centered on a temporally grounded WarmPrior mean rather than the origin, so the agent no longer searches the full prior from scratch and instead refines a local correction around an anchor that is already a competent action.
The RL policy also receives the prior mean $\mu$ as part of its augmented observation $\tilde{o}$. Since $\mu$ already encodes past chunks, appending it to the observation absorbs that dependency into the state so the RL problem stays Markovian, and it lets the residual $\Delta$ adapt to the current anchor.

\begin{figure}[t]
  \centering
  \begin{minipage}[b]{0.485\linewidth}
    \centering
    \includegraphics[width=\linewidth]{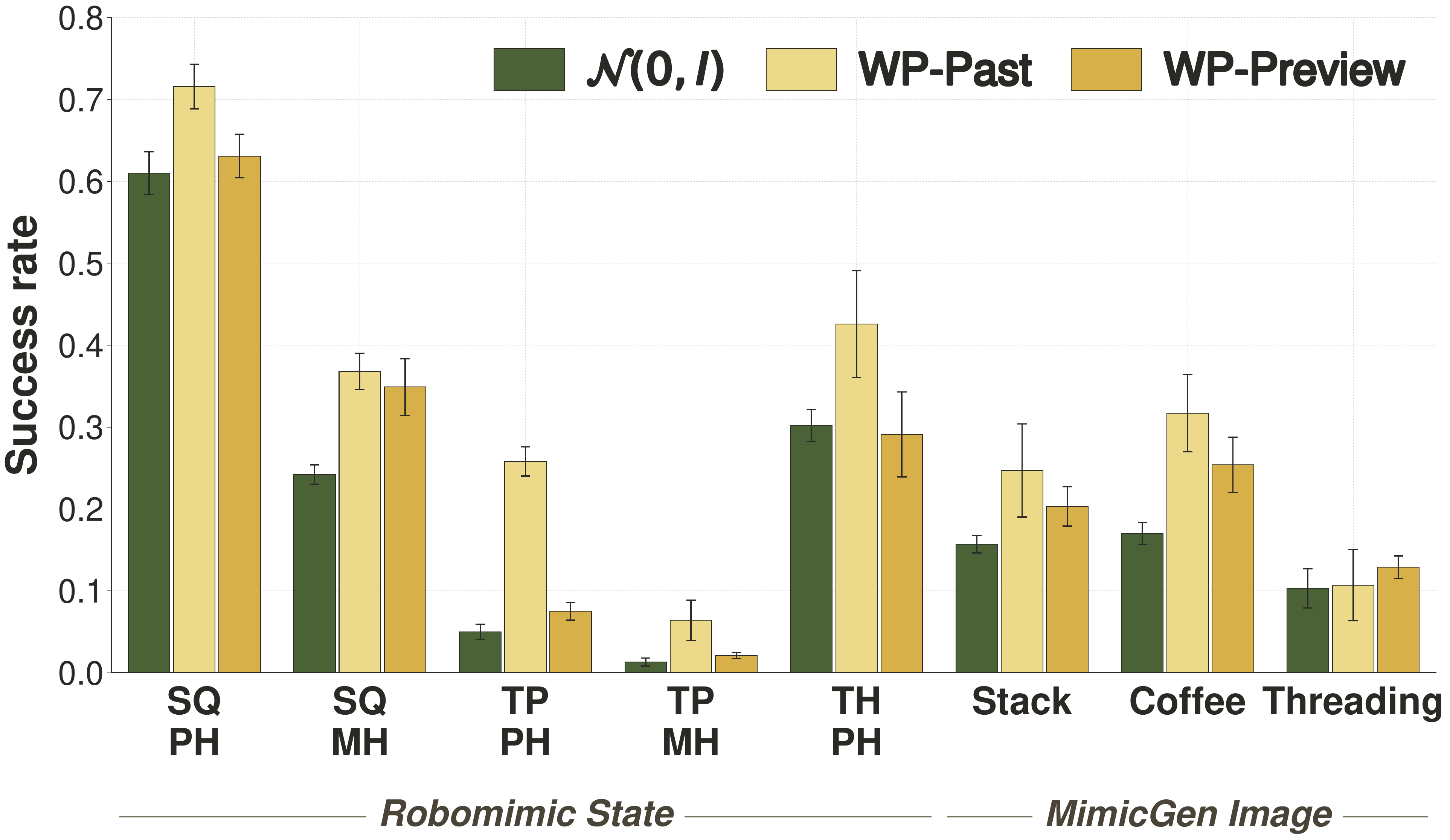}
    \caption{\textbf{Action-chunk length $H=1$ results (NFE$=1$).} Five Robomimic state tasks and three MimicGen image tasks.}
    \label{fig:ac1-bar}
  \end{minipage}%
  \hspace{1.0em}%
  \begin{minipage}[b]{0.485\linewidth}
    \centering
    \includegraphics[width=\linewidth]{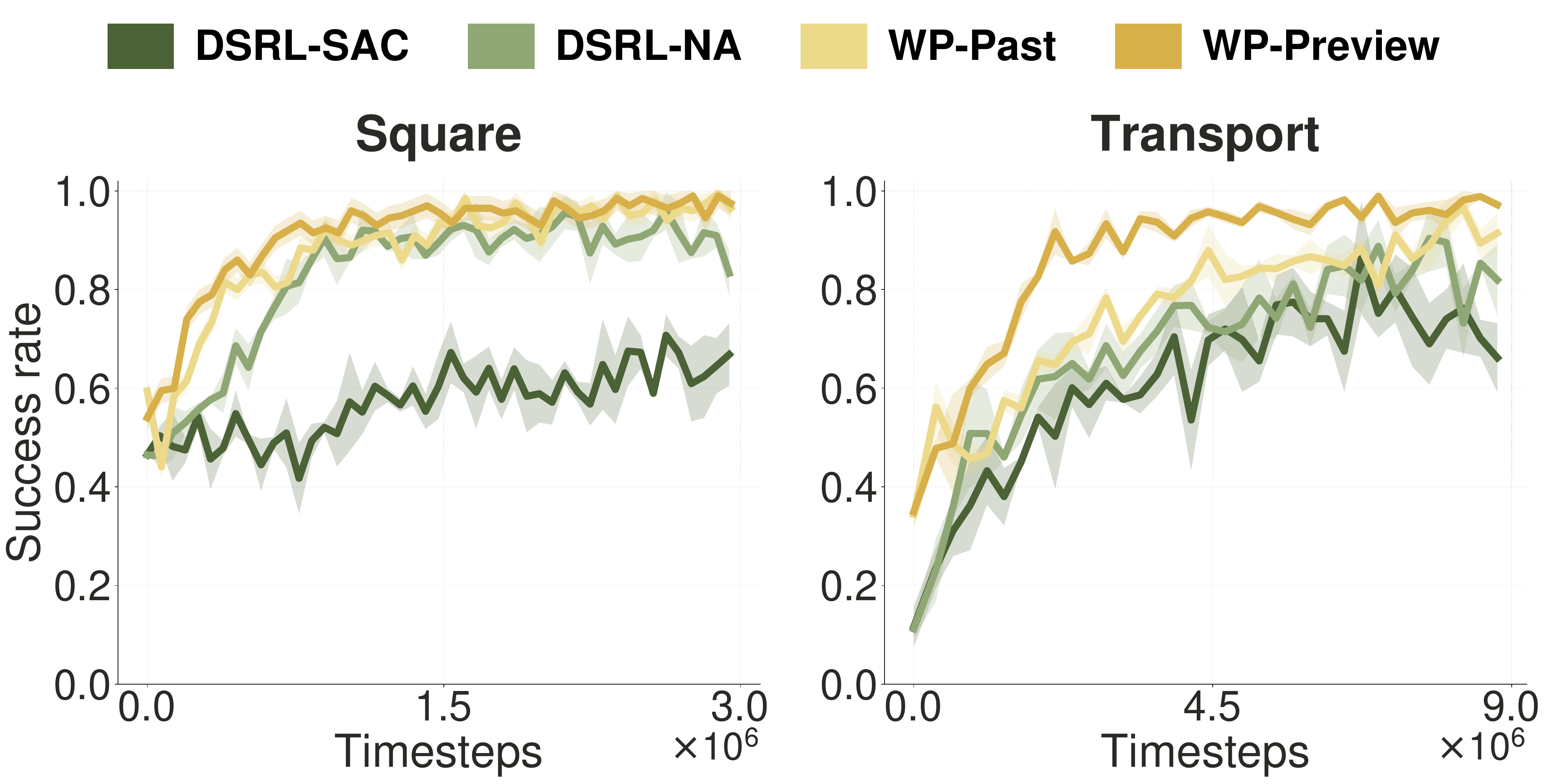}
    \caption{\textbf{Prior-space RL.} DSRL baselines vs.\ WarmPrior variants on Robomimic \textsc{Square} and \textsc{Transport}, averaged over 3 seeds ($\pm 1\sigma$ shading).}
    \label{fig:dsrl-vs-ours}
  \end{minipage}
\end{figure}

\paragraph{Setup.}
Among the Robomimic tasks, \textsc{Lift} and \textsc{Can} are already near-saturated under BC, so we run RL fine-tuning on \textsc{Square} and \textsc{Transport} with a frozen WarmPrior backbone pretrained by behavior cloning for 3000 epochs.
Our WP-Past and WP-Preview instantiate DSRL-NA with the conditioned residual of \eqref{eq:option5}, and we compare against vanilla DSRL-SAC and DSRL-NA as baselines.

\paragraph{Findings.}
\cref{fig:dsrl-vs-ours} shows that WP-Past and WP-Preview learn faster, converge more stably, and reach a higher asymptote than both DSRL baselines: both consistently exceed $0.99$ on \textsc{Square}, and on \textsc{Transport} they attain $\sim\!0.97$, while DSRL-NA and DSRL-SAC plateau around $0.9$.
To our knowledge, this is the first result to stably surpass $95\%$ success on \textsc{Transport}, the hardest of the Robomimic tasks, by RL fine-tuning of a flow-matching policy.
Because WarmPrior provides an efficient, semantically meaningful prior space centered on $\mu(o, h)$, searching over it is far more valuable than exploring an uninformed random noise space.

\section{Conclusion}
\label{sec:conclusion}

We revisited the source distribution of generative robotic policies and showed that replacing the uninformative $\mathcal{N}(0,I)$ with a temporally grounded WarmPrior consistently improves success rate on Robomimic, MimicGen, and a real Franka setup.
This single design choice straightens the learned flow in an OT-aligned sense, exposes a continuous $\sigma$-knob between within-rollout consistency and multimodal expressiveness, and shrinks the search space of prior-space RL on top of the pretrained policy.
Because WarmPrior leaves the network, interpolant, and loss untouched, we view the prior distribution as a new axis worth exploring in generative-policy design.

\clearpage
\bibliographystyle{plainnat}
\bibliography{example_paper}

%%%%%%%%%%%%%%%%%%%%%%%%%%%%%%%%%%%%%%%%%%%%%%%%%%%%%%%%%%%%
% APPENDIX
%%%%%%%%%%%%%%%%%%%%%%%%%%%%%%%%%%%%%%%%%%%%%%%%%%%%%%%%%%%%

\clearpage
\appendix

\section{Visualizing Success-Rate Uncertainty with Beta Posteriors}
\label{app:ac1}

We adopt the evaluation philosophy of~\citet{tri2025lbm}, which argues that single-number means with Gaussian error bars are an impoverished summary of policy performance and instead pushes for full posterior visualizations of the success-rate parameter. A seed-standard-error bar implicitly assumes the per-seed SR is symmetric and well approximated by a Gaussian; for a Bernoulli event near $0$ or $1$ the likelihood is skewed and the bar would cross an impossible boundary, and it conveys nothing about how \emph{overlapping} two methods' distributions are.
We therefore complement the bar charts in the main paper with Beta-posterior violins: for each (method, task) cell with $k$ successes out of $n = 200 \times 3 = 600$ rollouts (200 episodes per seed, 3 seeds), we visualize the posterior Beta$(k\!+\!1,\, n\!-\!k\!+\!1)$ under a uniform Beta$(1,1)$ prior. The violin width at $y$ is proportional to the posterior PDF, and the horizontal tick marks the posterior mean.
This makes three things easy to read off: (i) the plot is bounded to $[0,1]$ and skewed near the edges, so near-saturated and near-zero cells are rendered honestly; (ii) posterior \emph{overlap} between two methods is immediate, a more faithful proxy for significance than non-overlapping error bars; and (iii) high-variance cells (flatter violins) are visually distinguishable from confidently-estimated ones (tight violins).

One caveat: pooling all $600$ rollouts into a single Beta posterior treats them as i.i.d.\ Bernoulli draws from a common success probability, folding the three seeds' (different) policies into a single rate. This captures within-policy sampling uncertainty but absorbs across-seed (policy-level) variance into the same Bernoulli noise, so the posterior should be read as an estimate of the \emph{pooled} success rate; the seed-SE bars in the main paper remain the appropriate reference for method-level uncertainty.

\cref{fig:main-violin} shows this analysis for the main setting ($H=8$, NFE$=1$), and \cref{fig:ac1-violin} shows it when chunking is disabled ($H=1$, NFE$=1$).

\begin{figure}[H]
  \centering
  \includegraphics[width=0.95\textwidth]{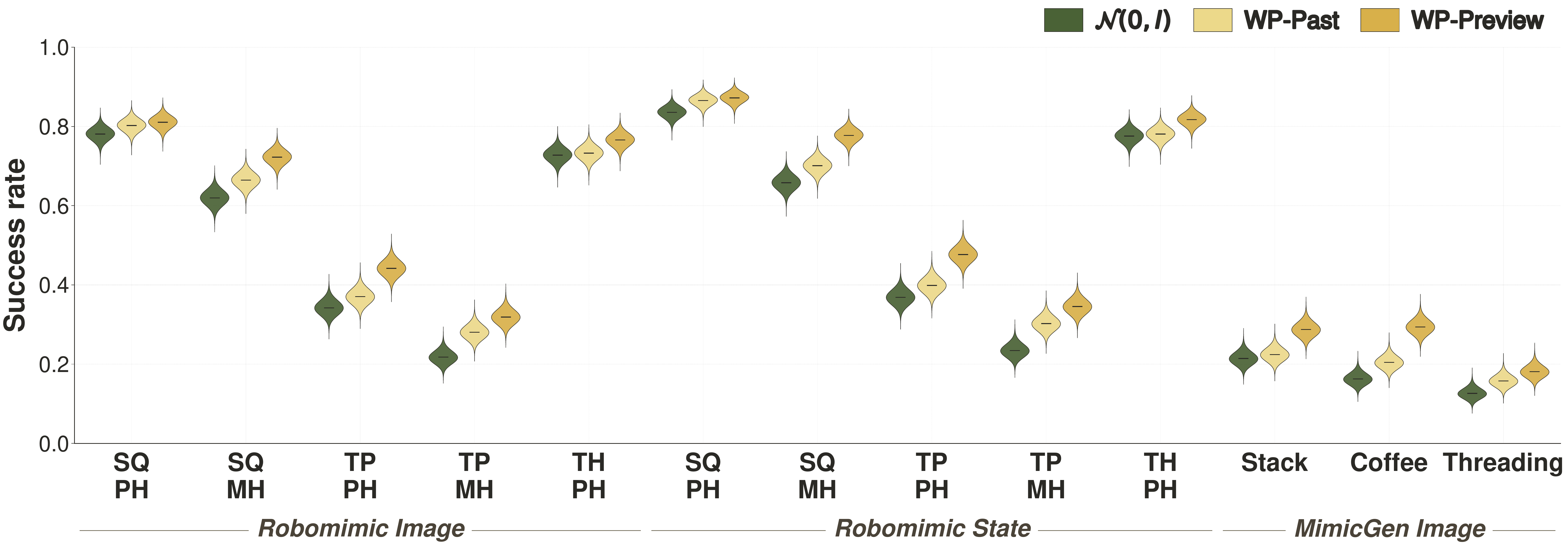}
  \caption{\textbf{Main results ($H=8$, NFE$=1$): Beta-posterior violins.} Same data as the NFE$=1$ column of \cref{tab:main-sr}; the violin width is proportional to the Beta$(k+1, n-k+1)$ posterior PDF over the success rate, with $n=600$ rollouts per cell.}
  \label{fig:main-violin}
\end{figure}

\begin{figure}[H]
  \centering
  \includegraphics[width=0.7\linewidth]{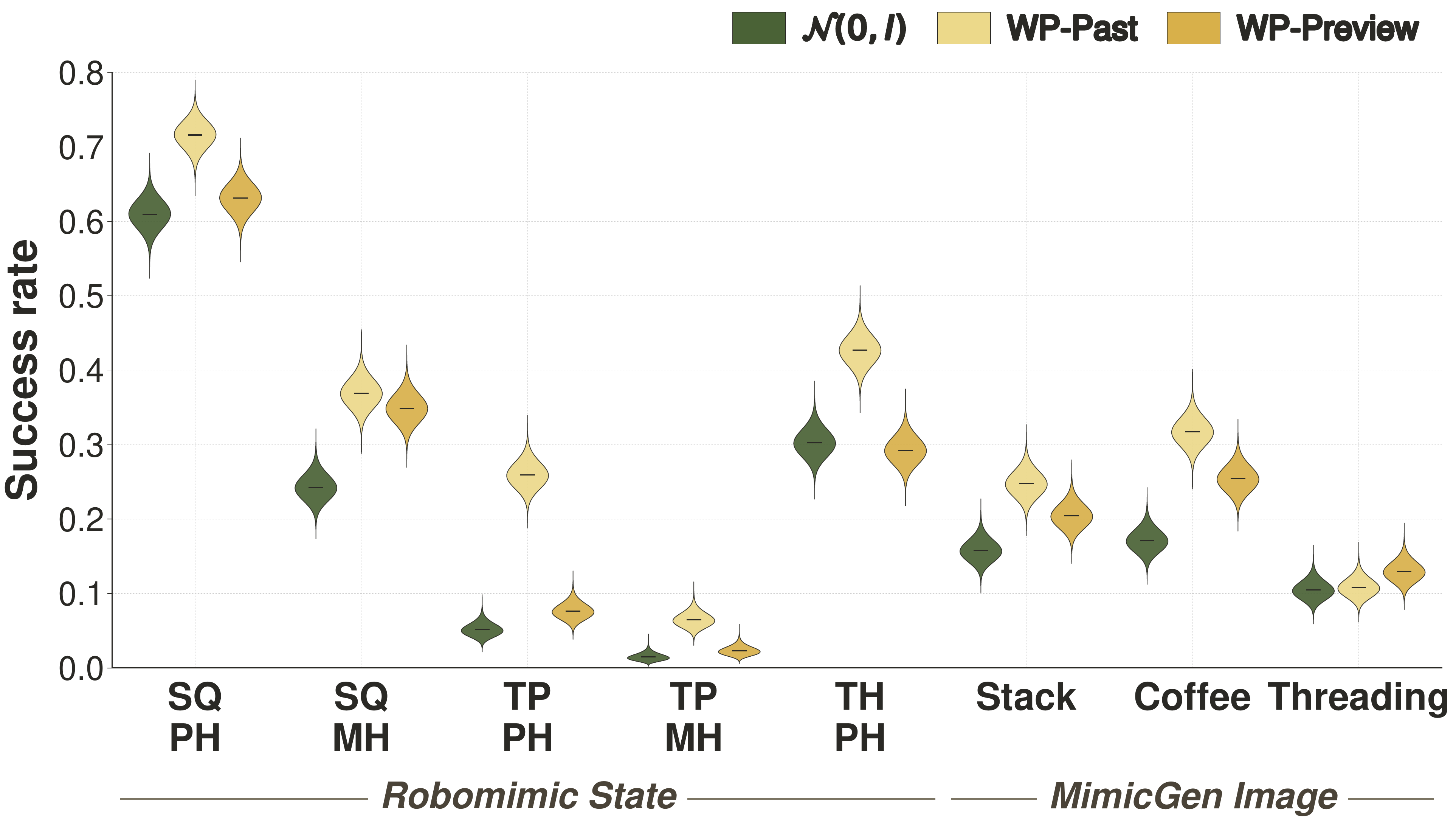}
  \caption{\textbf{Action-chunk length $H=1$ results (NFE$=1$): Beta-posterior violins.} Same data as \cref{fig:ac1-bar}; visualization is identical in style to \cref{fig:main-violin}.}
  \label{fig:ac1-violin}
\end{figure}

\section{Why WarmPrior Straightens Flows: A Branching-Cost Analysis}
\label{app:ot_detailed}

This appendix develops the theory underlying WarmPrior. We first formalize how endpoint ambiguity in the source-target coupling induces a branching cost in the learned flow (\cref{app:subsec-branching}), and then derive a bound showing how WarmPrior provably reduces this cost (\cref{app:subsec-warmprior-bound}).
For readability, we vectorize an action chunk into a single vector in \(\mathbb R^d\), where \(d = H d_a\), and condition throughout on the policy input \(o\).
We write \((A_0,A_1)\sim \Pi_o\) for the conditional joint law induced by the training procedure, where \(A_0\) is the source sample and \(A_1\) is the target action chunk.
Under the linear interpolant,
\begin{equation}
  A_t = (1-t)A_0 + tA_1,
  \qquad t\in[0,1].
  \label{eq:app-linear-interpolant}
\end{equation}

\paragraph{What the flow-matching loss is regressing.}
For the linear interpolant, the target velocity is
\[
  \dot A_t = A_1 - A_0.
\]
Hence the population flow-matching objective for a velocity field \(v_t(\cdot,o)\) is
\begin{equation}
  \mathcal L_o(v)
  \;\coloneqq\;
  \int_0^1
  \mathbb E\!\left[
    \|v_t(A_t,o) - (A_1-A_0)\|_2^2
    \,\middle|\,
    o
  \right]dt .
  \label{eq:app-pop-fm}
\end{equation}
This is simply an \(L^2\) regression problem: from the observable pair \((A_t,o)\), the network tries to predict the transport direction \(A_1-A_0\).

\subsection{The branching cost as endpoint ambiguity}
\label{app:subsec-branching}

The next theorem states that the irreducible error of this regression problem is exactly the conditional variance of the endpoint \(A_1\) given the intermediate point \(A_t\).
This is the sense in which path intersection or branching creates curvature: if many distinct endpoints are compatible with the same intermediate point, the vector field must average over them.

\begin{theorem}[Exact formula for the branching cost]
\label{thm:app-branching-cost}
For each \(t\in[0,1)\), the unique minimizer of \(\mathcal L_o(v)\) is
\begin{equation}
  v_t^\star(x,o)
  =
  \mathbb E[A_1\!-\!A_0 \mid A_t\!=\!x,o]
  =
  \frac{\mathbb E[A_1\mid A_t\!=\!x,o]-x}{1-t}.
  \label{eq:app-bayes-field}
\end{equation}
Define
\begin{equation}
  \mathcal B(o)
  \;\coloneqq\;
  \mathcal L_o(v^\star).
  \label{eq:app-def-branching}
\end{equation}
Then
\begin{equation}
  \mathcal B(o)
  =
  \int_0^1
  \frac{1}{(1\!-\!t)^2}
  \mathbb E\!\left[
    \|A_1\!-\!\mathbb E[A_1\!\mid\! A_t,o]\|_2^2
    \,\middle|\,
    o
  \right]dt .
  \label{eq:app-branching-formula}
\end{equation}
\end{theorem}

\paragraph{Proof sketch.}
The only information available to the predictor is \((A_t,o)\), so the best possible \(L^2\) predictor of the target velocity \(A_1-A_0\) is its conditional expectation given \((A_t,o)\).
The minimum mean-squared error is therefore the conditional variance of that target velocity.
For the linear interpolant, \(A_1-A_0 = (A_1-A_t)/(1-t)\), so this conditional variance can be rewritten directly in terms of the ambiguity of the endpoint \(A_1\) after observing the intermediate point \(A_t\).

\begin{proof}
Fix \(t<1\), and define
\[
  \Delta \;\coloneqq\; A_1 - A_0.
\]
The integrand of \eqref{eq:app-pop-fm} is an \(L^2\) regression problem:
among all \((A_t,o)\)-measurable square-integrable random variables, the unique minimizer of
\[
  g \;\mapsto\; \mathbb E[\|g-\Delta\|_2^2 \mid o]
\]
is the orthogonal projection of \(\Delta\) onto the sigma-field generated by \((A_t,o)\), namely
\[
  g^\star = \mathbb E[\Delta \mid A_t,o].
\]
Therefore
\[
  v_t^\star(A_t,o) = \mathbb E[A_1-A_0 \mid A_t,o].
\]
This proves the first equality in \eqref{eq:app-bayes-field}.

To prove the second equality, observe from \eqref{eq:app-linear-interpolant} that
\[
  A_t = (1-t)A_0 + tA_1
  \;\Longrightarrow\;
  A_1 - A_0 = \frac{A_1-A_t}{1-t}.
\]
Taking the conditional expectation given \(A_t=x\) and \(o\) yields
\[
  v_t^\star(x,o)
  =
  \mathbb E[A_1\!-\!A_0 \mid A_t\!=\!x,o]
  =
  \frac{\mathbb E[A_1\mid A_t\!=\!x,o]-x}{1-t},
\]
which proves \eqref{eq:app-bayes-field}.

We now compute the minimum value of the risk.
By the standard projection identity for conditional expectation,
\[
  \inf_g \mathbb E[\|g-\Delta\|_2^2\mid o]
  =
  \mathbb E[\|\Delta-\mathbb E[\Delta\mid A_t,o]\|_2^2 \mid o].
\]
Hence
\[
  \mathcal B(o)
  =
  \int_0^1
  \mathbb E\!\left[
    \|\Delta-\mathbb E[\Delta\mid A_t,o]\|_2^2
    \,\middle|\,
    o
  \right]dt.
\]
Using again \(\Delta=(A_1-A_t)/(1-t)\), we obtain
\[
  \Delta-\mathbb E[\Delta\mid A_t,o]
  =
  \frac{A_1-\mathbb E[A_1\mid A_t,o]}{1-t}.
\]
Substituting this into the previous display gives
\eqref{eq:app-branching-formula}.
\end{proof}

\paragraph{Interpretation.}
The quantity \(\mathcal B(o)\) is the irreducible flow-matching error under the coupling \(\Pi_o\).
Equation~\eqref{eq:app-branching-formula} shows that it is exactly the time-integrated ambiguity of the endpoint \(A_1\) after observing the intermediate point \(A_t\).
If \(A_t\) almost surely determines \(A_1\), then \(\mathcal B(o)=0\), meaning there is no branching ambiguity for the vector field to average over.
This is the ideal non-branching situation realized by OT-style Monge couplings.

\subsection{A detailed derivation of the WarmPrior bound}
\label{app:subsec-warmprior-bound}

We now prove the bound used in the main text.
Let \(P_{\mathcal W}\) denote the orthogonal projection onto the warm coordinates and let \(P_{\mathcal C}=I-P_{\mathcal W}\) be the projection onto the cold coordinates.
We write
\[
  d_{\mathcal W} \;\coloneqq\; \operatorname{tr}(P_{\mathcal W}),
\]
so that \(d_{\mathcal W}\) is exactly the number of warm scalar coordinates.

Under WarmPrior, the source sample takes the form
\begin{equation}
  A_0
  =
  P_{\mathcal W}(\mu + \sigma \Xi)
  +
  P_{\mathcal C}\Xi,
  \quad
  \Xi \sim \mathcal N(0,I_d),
  \label{eq:app-warmprior}
\end{equation}
where \(\Xi\) is conditionally independent of \((A_1,\mu)\) given \(o\).
The first term says that on the warm coordinates we start near a structured mean \(\mu\), while on the cold coordinates we keep the vanilla Gaussian prior.

To isolate the effect of the warm part, define the warm-coordinate branching cost by
\begin{equation}
  \mathcal B_{\mathcal W}(o)
  \;\coloneqq\;
  \int_0^1
  \frac{1}{(1\!-\!t)^2}
  \mathbb E\!\left[
    \|P_{\mathcal W}A_1 \!-\! \mathbb E[P_{\mathcal W}A_1\!\mid\! A_t,o]\|_2^2
    \,\middle|\,
    o
  \right]dt.
  \label{eq:app-warm-branching}
\end{equation}

\begin{proposition}[WarmPrior upper bound on the warm-coordinate branching cost]
\label{prop:app-warm-bound}
Under \eqref{eq:app-warmprior},
\begin{equation}
  \mathcal B_{\mathcal W}(o)
  \;\le\;
  \mathbb E\!\left[
    \|P_{\mathcal W}(A_1\!-\!\mu)\|_2^2
    \,\middle|\,
    o
  \right]
  +
  \sigma^2 d_{\mathcal W}.
  \label{eq:app-warm-bound}
\end{equation}
\end{proposition}

\paragraph{Proof sketch.}
The proof has one key idea.
The conditional expectation
\(\mathbb E[P_{\mathcal W}A_1\mid A_t,o]\)
is the \emph{best} predictor of the warm endpoint from \((A_t,o)\), so we may upper-bound its error by evaluating the same error at any simpler predictor.
We choose the very simple predictor \(P_{\mathcal W}A_t\), because it is directly observable from the interpolated sample and because, under the linear interpolant, the difference \(P_{\mathcal W}A_1-P_{\mathcal W}A_t\) contains an explicit factor of \((1-t)\) that exactly cancels the prefactor \(1/(1-t)^2\) in \eqref{eq:app-warm-branching}.
After this cancellation, the remaining expression splits into a mean-mismatch term and a Gaussian-noise term.

\begin{proof}
We proceed in three explicit steps.

\paragraph{Step 1: replace the optimal predictor with a tractable surrogate.}
For any square-integrable random variables \(Y\) and \(X\), the conditional expectation \(\mathbb E[Y\mid X]\) is the unique minimizer of
\[
  g \;\mapsto\; \mathbb E[\|Y-g(X)\|_2^2].
\]
Equivalently,
\begin{equation}
  \mathbb E\!\left[
    \|Y\!-\!\mathbb E[Y\!\mid\! X]\|_2^2
  \right]
  \le
  \mathbb E\!\left[
    \|Y\!-\!g(X)\|_2^2
  \right]
  \quad\text{for every measurable }g.
  \label{eq:app-best-predictor}
\end{equation}
We apply this with
\[
  Y = P_{\mathcal W}A_1,
  \;\;
  X = (A_t,o),
  \;\;
  g(A_t,o)=P_{\mathcal W}A_t.
\]
This choice is valid because \(P_{\mathcal W}A_t\) is clearly measurable with respect to \((A_t,o)\).
Using \eqref{eq:app-best-predictor} inside \eqref{eq:app-warm-branching} gives
\begin{equation}
  \mathcal B_{\mathcal W}(o)
  \le
  \int_0^1
  \frac{1}{(1\!-\!t)^2}
  \mathbb E\!\left[
    \|P_{\mathcal W}A_1 \!-\! P_{\mathcal W}A_t\|_2^2
    \,\middle|\,
    o
  \right]dt.
  \label{eq:app-bound-step1}
\end{equation}

\paragraph{Step 2: express the bound in the WarmPrior parameters \((\mu,\sigma)\).}
From \(A_t=(1-t)A_0+tA_1\), we have
\[
  P_{\mathcal W}A_t
  =
  (1\!-\!t)P_{\mathcal W}A_0 + tP_{\mathcal W}A_1.
\]
Hence
\begin{equation}
  P_{\mathcal W}A_1 - P_{\mathcal W}A_t
  =
  (1\!-\!t)\bigl(P_{\mathcal W}A_1 - P_{\mathcal W}A_0\bigr).
  \label{eq:app-step2a}
\end{equation}
Now substitute the WarmPrior form \eqref{eq:app-warmprior}:
on the warm coordinates,
\[
  P_{\mathcal W}A_0 = P_{\mathcal W}(\mu+\sigma\Xi).
\]
Therefore
\begin{equation}
  P_{\mathcal W}A_1 - P_{\mathcal W}A_0
  =
  P_{\mathcal W}(A_1\!-\!\mu) - \sigma P_{\mathcal W}\Xi.
  \label{eq:app-step2b}
\end{equation}
Combining \eqref{eq:app-step2a} and \eqref{eq:app-step2b},
\begin{equation}
  P_{\mathcal W}A_1 - P_{\mathcal W}A_t
  =
  (1\!-\!t)\Bigl(P_{\mathcal W}(A_1\!-\!\mu) - \sigma P_{\mathcal W}\Xi\Bigr).
  \label{eq:app-step2c}
\end{equation}

Squaring \eqref{eq:app-step2c} produces a \((1-t)^2\) factor that cancels the \(1/(1-t)^2\) prefactor in \eqref{eq:app-bound-step1}, leaving an integrand independent of \(t\). Substituting and integrating over \(t\in[0,1]\) yields
\begin{equation}
  \mathcal B_{\mathcal W}(o)
  \le
  \mathbb E\!\left[
    \bigl\|
      P_{\mathcal W}(A_1\!-\!\mu) - \sigma P_{\mathcal W}\Xi
    \bigr\|_2^2
    \,\middle|\,
    o
  \right].
  \label{eq:app-bound-step3}
\end{equation}

\paragraph{Step 3: decompose into mismatch and noise.}
Expand the squared norm:
\begin{equation}
  \bigl\|P_{\mathcal W}(A_1\!-\!\mu) - \sigma P_{\mathcal W}\Xi\bigr\|_2^2
  =
  \|P_{\mathcal W}(A_1\!-\!\mu)\|_2^2 + \sigma^2 \|P_{\mathcal W}\Xi\|_2^2
  - 2\sigma \bigl\langle P_{\mathcal W}(A_1\!-\!\mu),\, P_{\mathcal W}\Xi \bigr\rangle.
  \label{eq:app-expand-square}
\end{equation}
Taking the conditional expectation given \(o\), the cross term vanishes.
Indeed, by assumption, \(\Xi\) is conditionally independent of \((A_1,\mu)\) given \(o\), and has zero mean, so
\[
  \mathbb E\!\left[
    \bigl\langle P_{\mathcal W}(A_1\!-\!\mu),\, P_{\mathcal W}\Xi \bigr\rangle
    \,\middle|\,
    o
  \right]
  = 0.
\]
Therefore
\begin{equation}
  \mathbb E\!\left[
    \bigl\|P_{\mathcal W}(A_1\!-\!\mu) - \sigma P_{\mathcal W}\Xi\bigr\|_2^2
    \,\middle|\, o
  \right]
  =
  \mathbb E\!\left[
    \|P_{\mathcal W}(A_1\!-\!\mu)\|_2^2
    \,\middle|\, o
  \right]
  +
  \sigma^2 \mathbb E\!\left[\|P_{\mathcal W}\Xi\|_2^2\right].
  \label{eq:app-bound-step5}
\end{equation}
Since \(P_{\mathcal W}\) is the orthogonal projection onto a \(d_{\mathcal W}\)-dimensional subspace and \(\Xi\sim\mathcal N(0,I_d)\),
\[
  \mathbb E[\|P_{\mathcal W}\Xi\|_2^2] = d_{\mathcal W}.
\]
Substituting into \eqref{eq:app-bound-step5} and then into \eqref{eq:app-bound-step3} yields
\[
  \mathcal B_{\mathcal W}(o)
  \le
  \mathbb E\!\left[
    \|P_{\mathcal W}(A_1\!-\!\mu)\|_2^2
    \,\middle|\,
    o
  \right]
  +
  \sigma^2 d_{\mathcal W},
\]
which is exactly \eqref{eq:app-warm-bound}.
\end{proof}

\paragraph{Interpretation.}
Proposition~\ref{prop:app-warm-bound} shows that, on the warm coordinates, the branching cost is controlled by only two quantities:
\emph{(i)}~the mismatch between the WarmPrior mean \(\mu\) and the target \(A_1\), and
\emph{(ii)}~the residual Gaussian noise level \(\sigma\).
Thus, on the warm coordinates, WarmPrior becomes straighter when its mean is closer to the target and when its residual noise is smaller.
This explains the ordering of our variants.
For Preview, the training construction makes the warm mean target-aligned, so the mismatch term vanishes and only the \(\sigma^2 d_{\mathcal W}\) term remains.
For Past, the mean is only an approximation to the current target chunk, so an additional residual mismatch term remains.
The vanilla Gaussian baseline corresponds to a source mean that is far less aligned with the target, and therefore incurs a much larger ambiguity term.

\section{Training Details}
\label{app:training}

\cref{tab:training-hp} lists all training hyperparameters used in this work. Robomimic and MimicGen experiments share the same Diffusion Policy (ChiTransformer)~\cite{chi2023diffusionpolicy} backbone, which combines a Transformer~\cite{vaswani2017transformer} trunk with a ResNet-18~\cite{he2016resnet} image encoder, GroupNorm~\cite{wu2018groupnorm} normalization, and AdamW~\cite{loshchilov2019adamw} optimization; the two settings differ only in batch size and iteration count. For the real-robot experiments we fine-tune GR00T~N1.5-3B~\cite{nvidia2025gr00tn15,nvidia2025gr00tn1}, whose vision tower uses SigLIP-So400m~\cite{zhai2023siglip}, language backbone uses Qwen3-1.7B~\cite{qwen3} embedded in the Eagle~2.5-VL stack~\cite{eagle2025}, and action head uses a DiT~\cite{peebles2023dit} module; we keep the LLM and vision tower frozen and update only the action-head projector and the DiT module.

\begin{table}[p]
\centering
\caption{\textbf{Training hyperparameters across all experiments.} Robomimic and MimicGen use the ChiTransformer flow-matching backbone; real-robot experiments fine-tune GR00T~N1.5-3B with the LLM and vision tower frozen. ``---'' marks rows that do not apply to a given setup.}
\label{tab:training-hp}
\setlength{\tabcolsep}{3pt}
\resizebox{\linewidth}{!}{%
\begin{tabular}{@{}lccc@{}}
\toprule
& \textbf{Robomimic} & \textbf{MimicGen} & \textbf{GR00T~N1.5} \\
& (state / image) & (image) & (real Franka) \\
\midrule
\multicolumn{4}{l}{\emph{Architecture}} \\
\quad Backbone                      & ChiTransformer       & ChiTransformer       & GR00T~N1.5-3B \\
\quad Embedding dim                 & 384                  & 384                  & --- \\
\quad Transformer layers            & 8                    & 8                    & --- \\
\quad Attention heads               & 6                    & 6                    & --- \\
\quad Timestep emb. dim             & 128                  & 128                  & --- \\
\quad Attention dropout             & 0.1                  & 0.1                  & --- \\
\quad Image encoder                 & ResNet-18 (ImageNet) & ResNet-18 (ImageNet) & SigLIP-So400m (frozen) \\
\quad LLM                           & ---                  & ---                  & Qwen3-1.7B (frozen) \\
\quad VLM backbone                  & ---                  & ---                  & Eagle~2.5-VL \\
\quad Image input size              & $84{\times}84$       & $84{\times}84$       & $224{\times}224$ \\
\quad RGB cameras (image)           & \makecell{2 (1 TPV + 1 wrist) \\ (Transport: 2$\times$(1 TPV + 1 wrist))} & 2 (1 TPV + 1 wrist) & 3 (2 TPV + 1 wrist) \\
\quad Image augmentation            & \makecell{$76{\times}76$ random crop, \\ GroupNorm} & \makecell{$76{\times}76$ random crop, \\ GroupNorm} & \makecell{$0.95$-scale random crop, \\ resize to $224{\times}224$, \\ color jitter} \\
\quad State/action normalization    & per-key min--max     & per-key min--max     & per-key min--max \\
\quad Tuned components              & all                  & all                  & action-head projector + DiT \\
\midrule
\multicolumn{4}{l}{\emph{Optimization}} \\
\quad Optimizer                     & AdamW                & AdamW                & AdamW \\
\quad Learning rate                 & $1{\times}10^{-4}$   & $1{\times}10^{-4}$   & $1{\times}10^{-4}$ \\
\quad Weight decay                  & $1{\times}10^{-5}$   & $1{\times}10^{-5}$   & $1{\times}10^{-5}$ \\
\quad Adam $(\beta_1, \beta_2)$     & $(0.9,\,0.999)$      & $(0.9,\,0.999)$      & $(0.95,\,0.999)$ \\
\quad Adam $\epsilon$               & $1{\times}10^{-8}$   & $1{\times}10^{-8}$   & $1{\times}10^{-8}$ \\
\quad LR schedule                   & warmup + cosine      & warmup + cosine      & warmup + cosine \\
\quad Warmup ratio                  & 0.20                 & 0.20                 & 0.05 \\
\quad Gradient accumulation         & 1                    & 1                    & 1 \\
\quad Gradient checkpointing        & no                   & no                   & no \\
\quad Mixed precision               & FP32                 & FP32                 & bf16 + tf32 \\
\quad EMA rate                      & 0.995                & 0.995                & --- \\
\quad Batch size                    & 1024 / 256           & 128                  & 32 \\
\quad Iterations                    & 200{,}000            & 50{,}000             & 20{,}000 \\
\quad Training seeds                & 3                    & 3                    & 3 \\
\midrule
\multicolumn{4}{l}{\emph{Policy and data}} \\
\quad Action-chunk length $H$       & 8                    & 8                    & 16 \\
\quad Action dim                    & 7--14 (per task)     & 7--14 (per task)     & 7 \\
\quad Observation steps             & 2                    & 2                    & 1 \\
\quad State dim                     & 9--53 (per task)     & 9--53 (per task)     & 7 \\
\quad Demonstrations per task       & 250 (PH) / 300 (MH)  & 10                   & 30 \\
\quad Interpolant                   & linear               & linear               & linear \\
\quad Loss                          & flow matching        & flow matching        & flow matching \\
\quad WP-Past noise scale $\sigma$  & 0.5                  & 0.5                  & 0.5 \\
\quad WP-Preview noise scale $\sigma$ & 1.0                & 1.0                  & 1.0 \\
\midrule
\multicolumn{4}{l}{\emph{Evaluation}} \\
\quad Inference NFE                 & $\{1, 3, 9\}$        & $\{1, 3, 9\}$        & 4 \\
\quad Episodes per (task, seed)     & 200                  & 200                  & 50 \\
\quad Top-$K$ checkpoint averaging  & $K=3$                & $K=3$                & $K=1$ \\
\quad Parallel envs                 & 20                   & 20                   & --- (real) \\
\bottomrule
\end{tabular}%
}
\end{table}

\section{\texorpdfstring{$\sigma$}{sigma} Ablation}
\label{app:sigma-ablation}

The bound in \cref{eq:warmprior-branching-bound} predicts a non-monotone dependence on the prior std $\sigma$: too large and the irreducible $\sigma^2 d_{\mathcal W}$ term dominates, making the field bend to absorb a wide source; too small and the source concentrates onto the imperfect prior mean $\mu$ with no slack to absorb the persistence residual (WP-Past) or the forecast error (WP-Preview).
We empirically validate this trade-off on the most multimodal Robomimic task, \textsc{Square-MH}, by sweeping $\sigma \in \{1.5,\,1.0,\,0.5,\,0.3,\,0.1,\,0.05,\,0\}$ and evaluating each configuration with three seeds at $\text{NFE}=1$ and $H=8$.
\cref{fig:sigma-ablation} reports the resulting success rate and seed standard deviation.

\begin{figure}[t]
  \centering
  \begin{subfigure}{0.495\textwidth}
    \centering
    \includegraphics[width=0.95\linewidth]{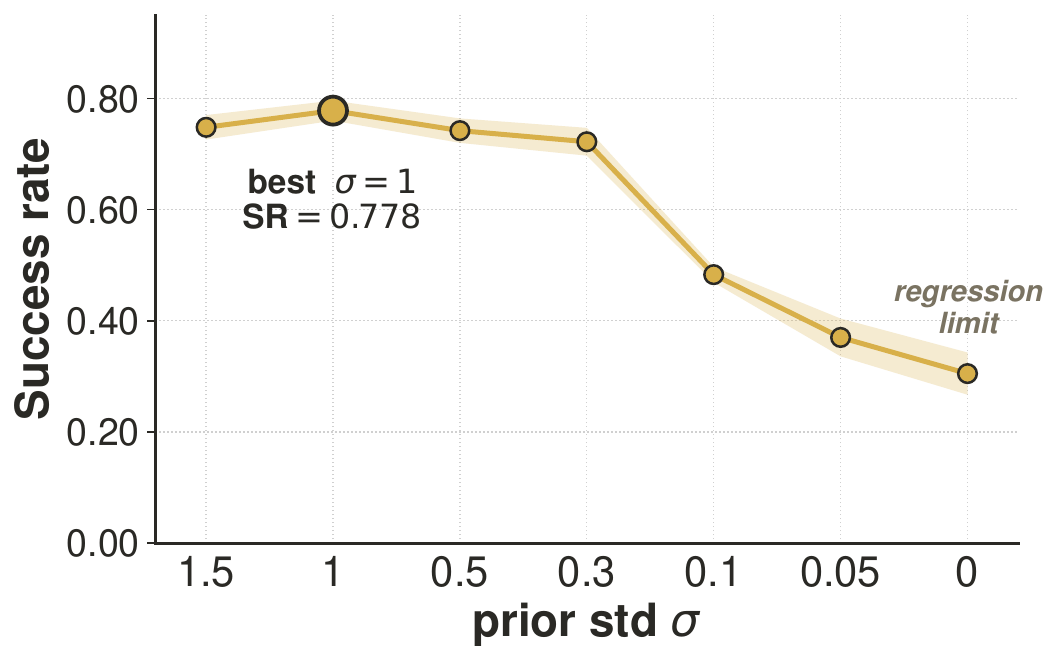}
    \caption{WP-Preview: peak at $\sigma=1.0$.}
    \label{fig:sigma-preview}
  \end{subfigure}\hspace{2pt}%
  \begin{subfigure}{0.495\textwidth}
    \centering
    \includegraphics[width=0.95\linewidth]{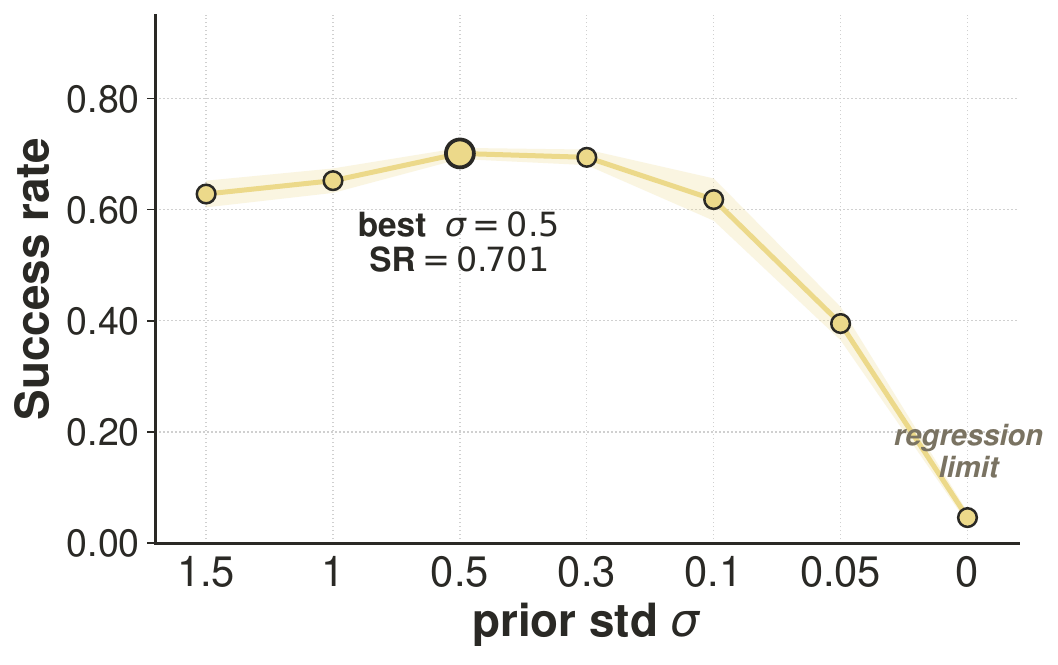}
    \caption{WP-Past: peak at $\sigma=0.5$.}
    \label{fig:sigma-past}
  \end{subfigure}
  \caption{\textbf{$\sigma$ ablation on Square-MH} ($\text{NFE}=1$, $H=8$, three seeds). Shaded band is $\pm 1$ seed std. The right end ($\sigma=0$) is the regression limit. The persistence prior of WP-Past carries more residual error than the WP-Preview forecast, so it benefits from a tighter source ($\sigma=0.5$ vs $\sigma=1.0$).}
  \label{fig:sigma-ablation}
\end{figure}

\paragraph{Findings.}
WP-Preview peaks at $\sigma=1.0$ with $\mathrm{SR}=0.778$, while WP-Past peaks at the smaller $\sigma=0.5$ with $\mathrm{SR}=0.701$.
This ordering is consistent with the role of $\mu$ in \cref{eq:warmprior-branching-bound}: Past carries the persistence residual $R$ in its mean-mismatch term, so concentrating the source onto $\mu$ via small $\sigma$ is more costly for Past than for Preview, and Past's optimum is pushed toward a smaller $\sigma$ where the $\sigma^2 d_{\mathcal W}$ penalty is reduced enough to compensate. Preview's smaller forecast error $E$ leaves the mean-mismatch term less sensitive to concentration, so its optimum sits where coverage of the multimodal target dominates the trade-off ($\sigma=1.0$).
Both curves exhibit a broad plateau followed by a sharp collapse: WP-Preview stays within $0.06$ of its optimum across $\sigma\in[0.3,1.5]$, and WP-Past stays within $0.08$ of its optimum across $\sigma\in[0.3,1.0]$, after which performance falls steeply for $\sigma\le 0.1$.
The plateau makes the choice of $\sigma$ forgiving in the moderate-noise regime.

\paragraph{Fixed $\sigma$ across tasks.}
Based on this ablation we fix $\sigma=1.0$ for WP-Preview and $\sigma=0.5$ for WP-Past for \emph{all} Robomimic, MimicGen, and real-robot experiments reported in the main paper.
\emph{We did not tune $\sigma$ per task.}
The plateau in \cref{fig:sigma-ablation} indicates that the method is robust to the choice of $\sigma$ in the moderate-noise regime, and the consistent gains obtained with these fixed values across eight benchmark tasks and a real-robot suite in \cref{tab:main-sr} confirm that a single setting transfers cleanly across embodiments and task difficulties without per-task tuning.

\paragraph{The $\sigma\!\to\!0$ limit.}
The right end of both curves ($\sigma=0$) corresponds to a deterministic source $a_0=\mu$, at which point the policy reduces to a regression-style mapping $\mu\mapsto a_1$ rather than a stochastic generative sampler.
This is the regime explored by A2A~\cite{jia2026a2a}, which encodes the action history into a deterministic latent source and composes a deterministic ODE on top.
Such a deterministic prior accelerates training convergence because the source is no longer randomized, but it also collapses the conditional $p(a\mid o)$ to a single mode, giving up the multimodal coverage that motivates generative imitation in the first place.
\cref{fig:sigma-ablation} makes this concrete: the $\sigma=0$ end-point drops to $\mathrm{SR}\approx 0.31$ for Preview and to $\mathrm{SR}\approx 0.05$ for Past, an essentially complete failure.
The Past collapse is the more severe of the two because its prior mean is the previously executed chunk: without injected noise the policy is asked to map the past chunk directly to the next chunk through a network that never saw such a deterministic source--target pairing during training.
The full WarmPrior with $\sigma>0$ retains the generative structure while still exploiting the temporally grounded prior, and our ablation shows that this stochastic regime is where the success rate is maximized.

\section{Comparing WarmPrior with Real-Time Chunking}
\label{app:rtc-compare}

WarmPrior and Real-Time Chunking (RTC)~\cite{black2025rtc} both exploit the fact that the previously executed action chunk carries a great deal of information about the next one, yet they intervene at different points in the policy stack. RTC is an inference-time procedure: at each new decision step the policy regenerates the next chunk while clamping its early positions to the actions still being executed, so the freshly generated chunk is forced to commit to the same mode as the one it overlaps with. WarmPrior is a training-time prior shaping mechanism (\cref{sec:method}): it anchors the source distribution on the previous chunk and trains the velocity field under that anchored coupling, so the learned flow itself is shorter and straighter without any inference-time inpainting. Because both mechanisms read from the same ``past chunk'' signal, it is natural to ask whether they are merely two encodings of the same gain. This section disentangles the two.

\paragraph{Setup.}
This experiment uses a different backbone from \cref{sec:main-real}: we evaluate on the $\pi_{0.5}$ vision-language-action model~\cite{pi05}, whose flow-matching action head is the natural backbone to test alongside RTC. The remainder of the real-robot pipeline matches \cref{sec:main-real}: a Franka Research 3 with teleoperated demonstrations collected on the DROID platform~\cite{khazatsky2024droid}, three training seeds, and 20 evaluation trials per seed. We deliberately pick two tasks where RTC is known to help, namely the dynamic and precision-sensitive \emph{Block Throwing} and \emph{Towel Folding} (\cref{fig:rtc-compare-task}), and ask whether WarmPrior also gains in this regime.

We compare four configurations of the same flow-matching backbone:
\textbf{Base} (vanilla $\mathcal N(0,I)$ prior, independent per-chunk inference),
\textbf{RTC} (the Base policy executed under the real-time chunking inference procedure),
\textbf{WarmPrior} (WP-Preview with the temporally grounded prior, independent per-chunk inference), and
\textbf{RTC+WarmPrior} (the WP-Preview policy executed under the same RTC procedure).
The combined configuration is the natural ``stack'' of the two interventions: WP-Preview reshapes $p_0$ at training time, and RTC additionally clamps the executing portion of the trajectory at inference time.

\begin{figure}[t]
  \centering
  \begin{minipage}[b]{0.45\linewidth}
    \centering
    \includegraphics[width=\linewidth]{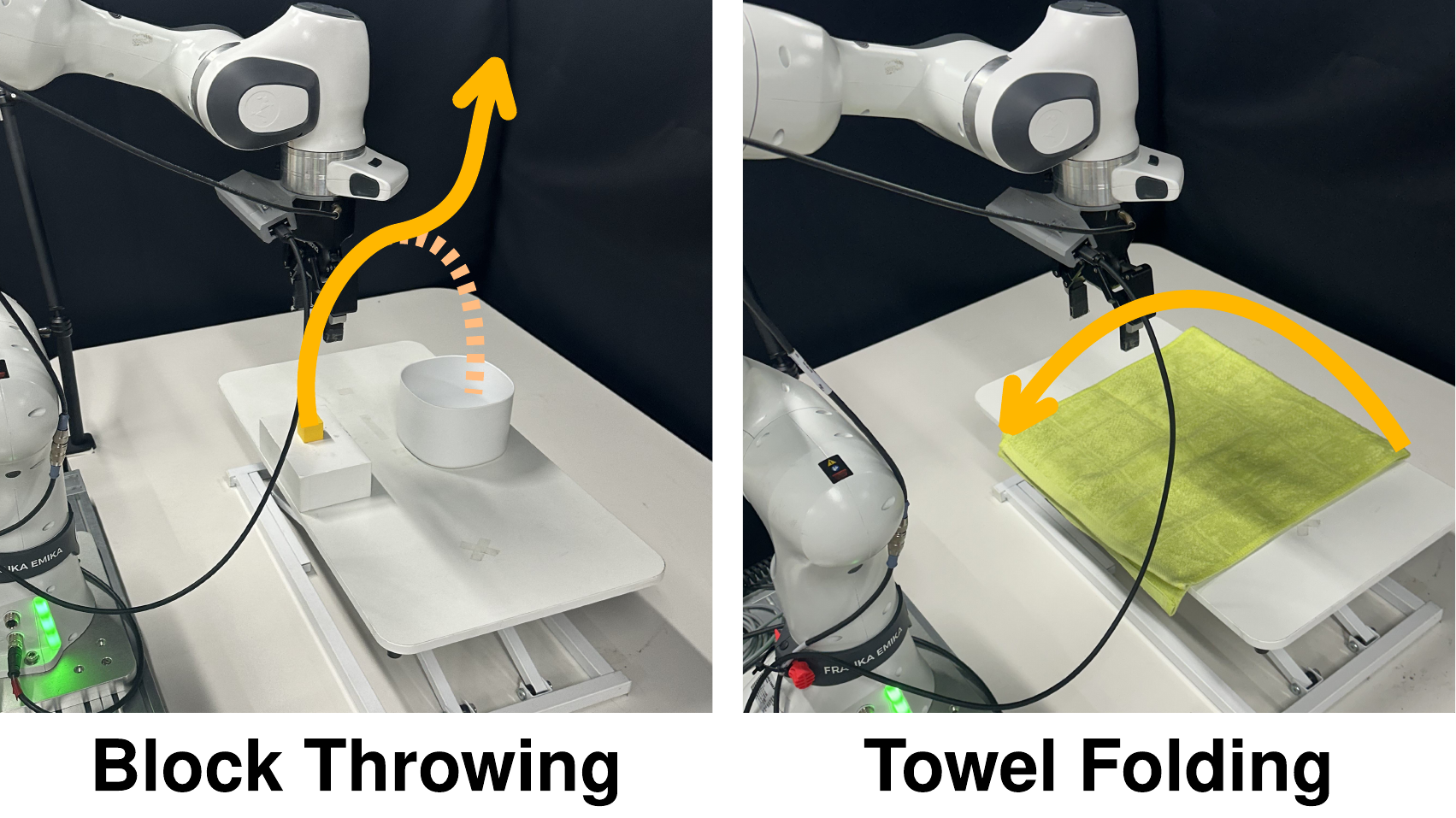}
    \caption{\textbf{RTC comparison tasks.} The two highly dynamic real-robot scenes used in \cref{fig:rtc-compare-bar}: \emph{Block Throwing} and \emph{Towel Folding}. Both involve fast, committed whole-arm motions where mode-switching across chunk boundaries is particularly visible.}
    \label{fig:rtc-compare-task}
  \end{minipage}\hspace{1.0em}%
  \begin{minipage}[b]{0.51\linewidth}
    \centering
    \includegraphics[width=\linewidth]{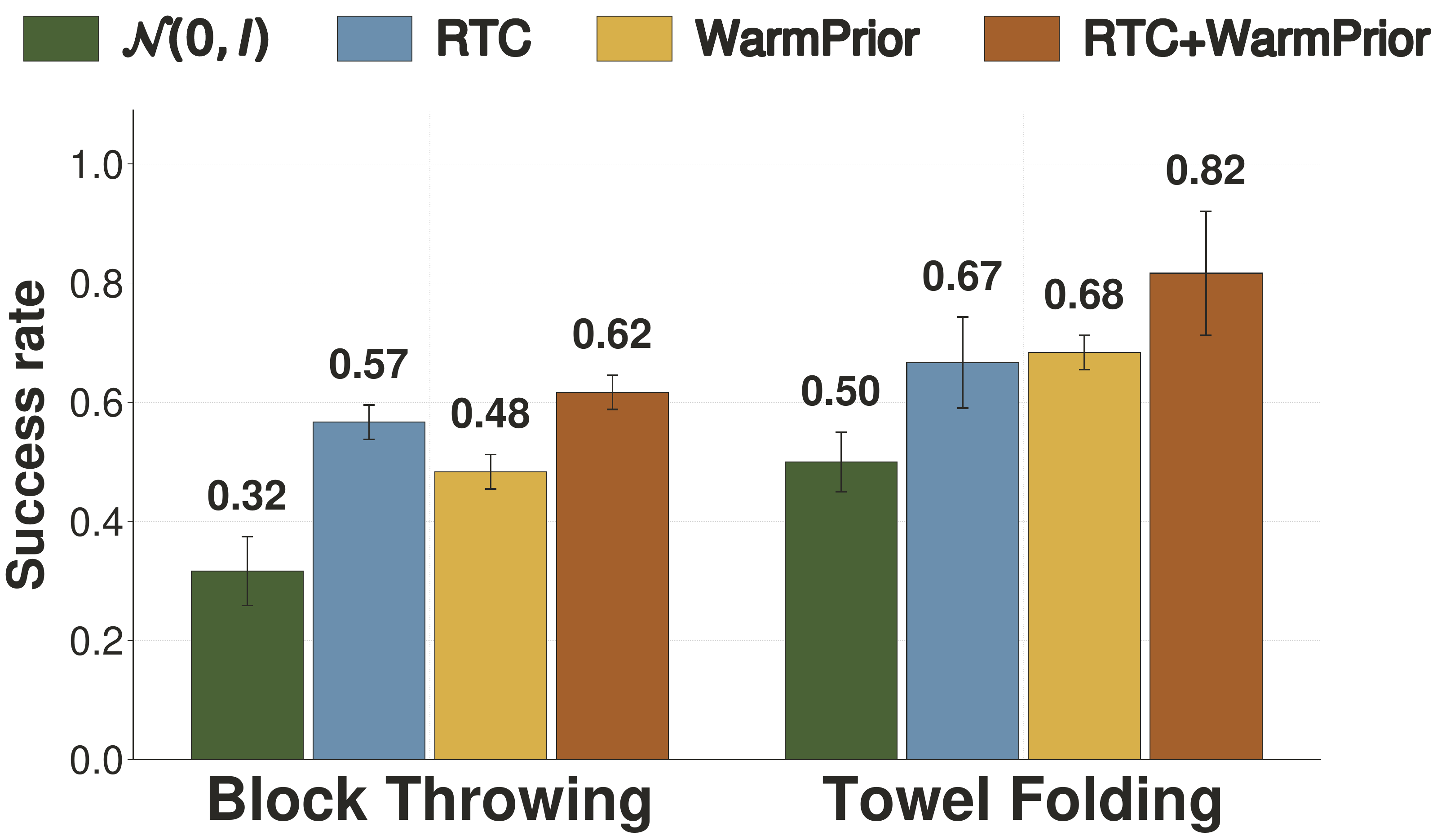}
    \caption{\textbf{RTC vs.\ WarmPrior on highly dynamic tasks.} Real-robot success rate of $\pi_{0.5}$ (mean and seed standard deviation over three training seeds, 20 trials per seed). RTC and WarmPrior each improve over the baseline, and the combination exceeds both, suggesting their gains come from distinct mechanisms.}
    \label{fig:rtc-compare-bar}
  \end{minipage}
\end{figure}

\paragraph{Findings.}
\cref{fig:rtc-compare-bar} reports per-task success rate, and three observations follow.

\emph{(i) RTC is effective on highly dynamic tasks.} RTC nearly doubles the baseline on \emph{Block Throwing} ($0.32 \!\to\! 0.57$) and lifts \emph{Towel Folding} from $0.50$ to $0.67$. This is consistent with the picture in which mode-switching across chunk boundaries is most damaging when the underlying motion is fast and committed, exactly the regime where RTC's explicit inpainting suppresses inter-chunk discontinuities.

\emph{(ii) WarmPrior also provides consistent gains.} WarmPrior alone improves both tasks ($0.32 \!\to\! 0.48$ on \emph{Block Throwing}, $0.50 \!\to\! 0.68$ on \emph{Towel Folding}), and on \emph{Towel Folding} its gain is comparable to that of RTC ($0.68$ vs $0.67$). This is the picture predicted by \cref{sec:straight}: the prior mean drawn from the previous chunk reduces endpoint ambiguity for the velocity field on \emph{both} tasks, with the bound of \cref{eq:warmprior-branching-bound} tightened by exactly the same temporally grounded signal that RTC also exploits.

\emph{(iii) Combining the two yields an additional improvement.} \emph{RTC+WarmPrior} reaches $0.62$ on \emph{Block Throwing} and $0.82$ on \emph{Towel Folding}, exceeding both individual methods on both tasks; the increment is largest on \emph{Towel Folding}, where the combination ($0.82$) sits well above either RTC alone ($0.67$) or WarmPrior alone ($0.68$). If the two methods relied on the same underlying effect, stacking them would saturate and produce no further gain. The fact that they compound is evidence that they reach their success rate via distinct mechanisms.
We summarize the picture as follows.
\textbf{RTC} enforces \emph{explicit mode commitment} at inference time: by clamping the early portion of the flow to the chunk currently being executed, it guarantees zero discontinuity at the boundary, which is what stabilizes fast, committed motions across chunk transitions.
\textbf{WarmPrior} reshapes the training-time coupling so that the learned velocity field is itself straighter, in the OT-aligned sense quantified by \cref{tab:curvature} and bounded in \cref{eq:warmprior-branching-bound}.
The two interventions address different failure modes of standard flow matching, namely curved learned flows (training side) and inter-chunk discontinuities (inference side), so combining them removes both at once.

\paragraph{Practical implications.}
A practical consequence: RTC requires chunks long enough to leave a meaningful overlap window, and it commits the policy to the chunk currently being executed before re-planning, which is awkward on tasks that demand fast within-chunk reactivity (\cref{sec:ac1}). WarmPrior's $\sigma$ knob (\cref{eq:warmprior-branching-bound}, \cref{sec:ac1}) instead supplies a continuous trade-off between \emph{temporal commitment} and \emph{multimodal expressiveness} that remains operative even at $H = 1$, where action chunking is effectively disabled. WarmPrior should therefore be read as a complement to RTC when chunking is available, and as a viable alternative when it is not.

% NeurIPS 2026 requires the paper checklist for full submissions.
% Place the checklist file at neurips2026/checklist.tex (the default template
% provides one) and uncomment the line below before submitting.
% \newpage
% \input{checklist.tex}

\end{document}